\documentclass[10pt,twocolumn,letterpaper]{article}

\usepackage[pagenumbers]{cvpr} 

\definecolor{cvprblue}{rgb}{0.21,0.49,0.74}
\usepackage[pagebackref,breaklinks,colorlinks,allcolors=cvprblue]{hyperref}

\usepackage[utf8]{inputenc} 
\usepackage[T1]{fontenc}    
\usepackage{url}            
\usepackage{booktabs}       
\usepackage{amsfonts}       
\usepackage{nicefrac}       
\usepackage{microtype}      
\usepackage{graphicx}
\usepackage[table]{xcolor}
\usepackage{xcolor}         
\definecolor{Gray}{gray}{0.9}
\usepackage{threeparttable}
\usepackage{amssymb}
\usepackage{pifont}

\usepackage{amsmath}
\usepackage{amsthm}
\usepackage{array}
\usepackage{siunitx}
\usepackage{xspace}
\usepackage{longtable}
\usepackage{marvosym}
\usepackage{enumerate}
\usepackage{forest}
\usepackage{mathrsfs}
\usepackage{arydshln}
\usepackage{wrapfig}

\newtheorem{definition}{Definition}
\usepackage{makecell}
\usepackage{thmtools} 
\usepackage{thm-restate}

\usepackage{algorithm}
\usepackage{algorithmic}

\usepackage{enumitem}
\usepackage{adjustbox}

\usepackage{multirow}

\makeatletter
\newcommand*\bigcdot{\mathpalette\bigcdot@{.5}}
\newcommand*\bigcdot@[2]{\mathbin{\vcenter{\hbox{\scalebox{#2}{$\m@th#1\bullet$}}}}}
\makeatother

\usepackage[capitalize]{cleveref}
\crefname{section}{Sec.}{Secs.}
\Crefname{section}{Section}{Sections}
\Crefname{table}{Table}{Tables}
\crefname{table}{Tab.}{Tabs.}

\newcommand{\calX}{\mathcal{X}}

\newcommand{\bx}{\mathbf{x}}

\def\paperID{} 
\def\confName{CVPR}
\def\confYear{2026}

\title{Adaptive Confidence Regularization for Multimodal Failure Detection}

\author{
Moru Liu$^1$, Hao Dong$^2$, Olga Fink$^3$, Mario Trapp$^{1, 4}$\\
$^1$Technical University of Munich \quad $^2$ETH Z\"urich \quad $^3$EPFL \quad $^4$Fraunhofer IKS \\
}

\begin{document}
\maketitle

\begin{abstract}
The deployment of multimodal models in high-stakes domains, such as self-driving vehicles and medical diagnostics, demands not only strong predictive performance but also reliable mechanisms for detecting failures. In this work, we address the largely unexplored problem of failure detection in multimodal contexts. We propose Adaptive Confidence Regularization (ACR), a novel framework specifically designed to detect multimodal failures. Our approach is driven by a key observation: in most failure cases, the confidence of the multimodal prediction is significantly lower than that of at least one unimodal branch, a phenomenon we term \textit{confidence degradation}. To mitigate this, we introduce an \textit{Adaptive Confidence Loss} that penalizes such degradations during training. In addition, we propose \textit{Multimodal Feature Swapping}, a novel outlier synthesis technique that generates challenging, failure-aware training examples. By training with these synthetic failures, ACR learns to more effectively recognize and reject uncertain predictions, thereby improving overall reliability. 
Extensive experiments across four datasets, three modalities, and multiple evaluation settings demonstrate that ACR achieves consistent and robust gains. The source code will be available at \href{https://github.com/mona4399/ACR}{https://github.com/mona4399/ACR}.
\end{abstract}

\section{Introduction} \label{sec:introduction}
Multimodal models are increasingly adopted in safety-critical domains such as autonomous driving and medical diagnostics~\cite{dong2025mmdasurvey,sun2025unseen,fink2026physics}. By integrating complementary cues from diverse modalities (e.g., video, audio), they often achieve superior robustness and generalization over unimodal approaches~\citep{feng2020deep, wang2018tienet}. 
However, even state-of-the-art models can be dangerously overconfident in their erroneous predictions~\citep{zeng2025towards}, posing serious risks in high-stakes applications. In such settings, detecting untrustworthy predictions is as crucial as achieving high overall accuracy. While prior work in uncertainty estimation~\citep{lakshminarayanan2017simple}, calibration~\citep{pmlr-v70-guo17a}, and out-of-distribution (OOD) detection~\citep{liu2020energy} has aimed to mitigate overconfidence, these methods often fail to reliably flag individual predictions that should be rejected. Failure detection (FD) -- also referred to as misclassification detection or selective classification -- directly addresses this challenge by identifying unreliable predictions for potential rejection or human intervention, thereby reducing the risk of catastrophic failures~\citep{feng2022towards}. 

While FD is well-established in unimodal settings, with methods spanning confidence-based scoring~\citep{granese2021doctor,jiang2018trust}, outlier exposure~\citep{cheng2024breaking,OpenMix}, and confidence learning~\citep{ConfidNet,moon2020confidence}, its extension to multimodal systems remains largely unexplored. This gap is non-trivial, as unimodal approaches often fail to effectively leverage the complementary information across  modalities or to handle failure modes unique to multimodal data, such as signal conflict and misalignment~\citep{rasenberg2020alignment}. Furthermore, some works~\citep{dong2024multiood,li2024dpu} explore OOD detection with multiple modalities, but their settings fundamentally differ from those of FD.
To illustrate  the potential benefits of utilizing multiple modalities for FD, we present empirical results on the HMDB51 dataset~\citep{kuehne2011hmdb}. All models in this analysis were trained solely with a standard cross-entropy loss.  As shown in Figure~\ref{fig:motivation0} (left), a simple fusion of video and optical flow inputs substantially improves FD performance -- measured by AURC, AUROC, and FPR95 -- over unimodal baselines. This finding highlights \textit{the considerable potential of multimodal signals for improving FD}. Concurrently, Figure~\ref{fig:motivation0} (right) reveals that sophisticated OOD detection methods like Energy~\citep{liu2020energy}, Entropy~\citep{tian2022pixel}, and MaxLogit~\citep{hendrycks2019scaling} are outperformed by a simple Maximum Softmax Probability (MSP) baseline~\citep{hendrycks2016baseline}. Taken together, these findings demonstrate that \textit{merely adapting OOD techniques is insufficient} and motivate the development of dedicated methods tailored for multimodal FD. 

\begin{figure*}[t]
  \centering
  \includegraphics[width=0.9\linewidth]{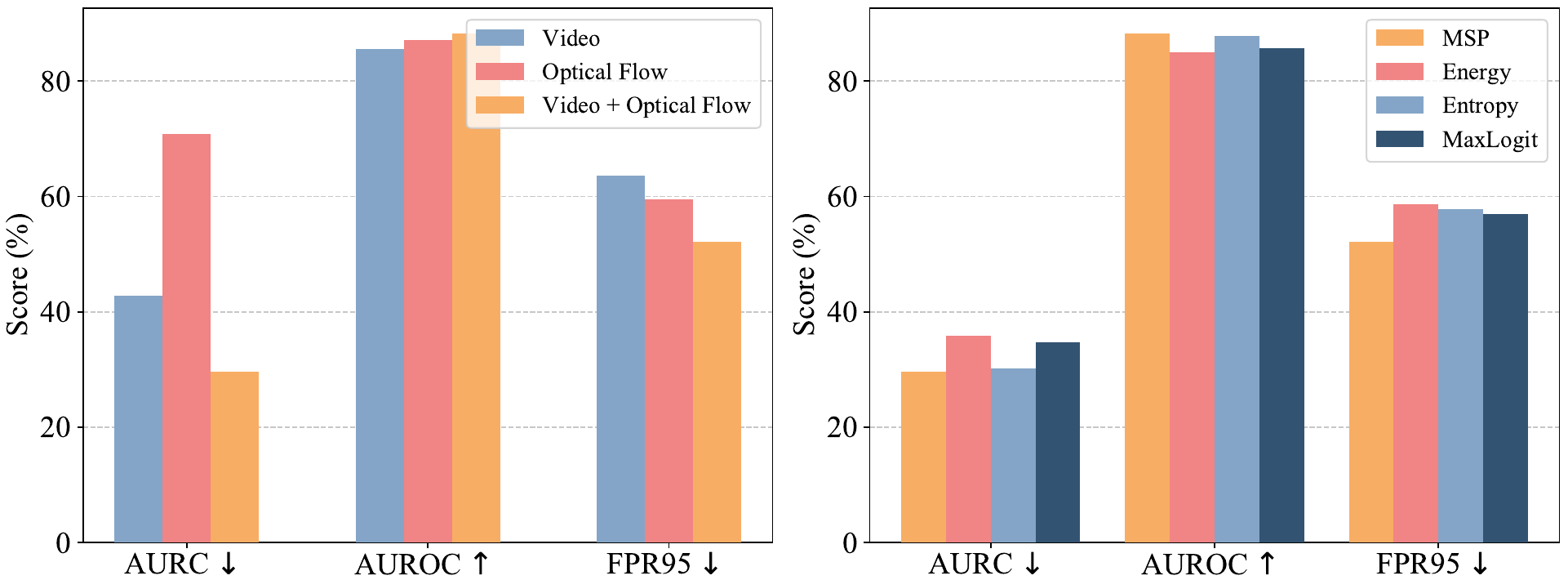}
   \vspace{-0.2cm}
   \caption{
    {(Left)} Multimodal models substantially enhance FD performance compared to unimodal models, without the need for complex designs. 
    {(Right)} Advanced OOD detection methods underperform on FD tasks, while the simple MSP baseline surprisingly remains the most effective.
}
   \label{fig:motivation0}
\end{figure*}


In this work, we identify and systematically characterize the phenomenon of \textit{confidence degradation}, a scenario where the confidence of fused multimodal predictions undesirably falls below that of individual unimodal predictions, particularly in misclassified instances. To address this, we propose Adaptive Confidence Regularization (ACR), the first dedicated framework for detecting failures (i.e., misclassifications) in multimodal systems. ACR comprises two key innovations: (1) an \textit{Adaptive Confidence Loss} that explicitly penalizes confidence degradation during training, and (2) \textit{Multimodal Feature Swapping}, a novel augmentation technique that synthesizes challenging, failure-aware training samples by swapping cross-modal embeddings. Training with the confidence penalty and failure-aware outliers improves the model’s ability to detect and reject uncertain samples, yielding gains in both accuracy and FD performance. Comprehensive experiments across five datasets and five modalities demonstrate that ACR sets a new state of the art, outperforming prior best methods by up to $9.58\%$ in AURC, $1.63\%$ in AUROC, and $15.45\%$ in FPR95. Further ablation studies under distribution shifts and multimodal OOD detection settings confirm the robustness and strong generalization of our approach. The primary contributions of this work are:
\begin{itemize}[leftmargin=*]
\setlength\itemsep{0em}
\item We highlight the importance of leveraging multimodal inputs for effective FD, and provide empirical evidence on the limitations of existing OOD detection approaches in this context. 
\item We reveal and empirically validate the phenomenon of \textit{confidence degradation} in multimodal models, showing its strong correlation with prediction failures.
\item We propose ACR, the first dedicated framework tailored to the complex task of multimodal FD. ACR integrates a novel Adaptive Confidence Loss, addressing the issue of \emph{confidence degradation}, and introduces  Multimodal Feature Swapping to further enhance confidence reliability. 
\item We perform  extensive evaluations  across diverse datasets and modalities, demonstrating the robustness and effectiveness of ACR in a wide range of scenarios.
\end{itemize}


\section{Methodology}
\subsection{Problem Setup}
\noindent\textbf{Multimodal Failure Detection} aims to detect misclassified samples using \textbf{\emph{multiple modalities}}.  {We consider a training set $\mathbb{D} = \{(\bx_i, y_i)\}_{i=1}^n$ drawn \emph{i.i.d.} from the joint data distribution $P_{\mathcal{X}\mathcal{Y}}$, where $\mathcal{X}$ is the input space and $\mathcal{Y}=\{1,2,...,C\}$ is the label space. Each sample $\mathbf{x}_i$ is composed of $M$ modalities, denoted as $\mathbf{x}_i=\{x_i^k \mid k=1,\cdots,M\}$. 
Let $f: \calX \mapsto \mathbb{R}^{C}$ be a neural network trained on samples in $P_{\mathcal{X}\mathcal{Y}}$ that predicts the label of each input sample. The $f$ in multimodal failure detection comprises $M$ feature extractors $g_k(\cdot)$ and a classifier $h(\cdot)$.} Each feature extractor $g_k(\cdot)$ extracts an embedding $\mathbf{E}^{k}$ for its corresponding modality $k$, and the classifier $h(\cdot)$ takes the combined embeddings from all modalities as input and outputs a prediction probability~$\hat{p}$:
\begin{equation}
\begin{split}
  \hat{p} = \delta(f(\mathbf{x})) 
  &= \delta(h([g_1(x^1), ..., g_M(x^M)]))  \\
   &=  \delta(h([\mathbf{E}^{1}, ..., \mathbf{E}^{M}])) ,
  \label{eqn:pred}
\end{split}
\end{equation}
where $\delta(\cdot)$ is the softmax function. 
We further include a classifier $h_k(\cdot)$ for each modality $k$ to get predictions from each modality separately, with the prediction probability from the $k$-th modality as $\hat{p}^k = \delta(h_k(g_k(x^k)))$.

To safely deploy classifier $f$ in real-world applications, it should not only be able to make accurate predictions but also \textit{distinguish and reject incorrect ones}. Formally, let $\kappa(\cdot)$ be a confidence-scoring function that quantifies the model's confidence in its prediction. With a predefined threshold $\tau \in \mathbb{R}^{+}$, the misclassified samples can be detected based on a decision function $G$ such that for a given input~$\mathbf{x}$:
\begin{equation}
\label{eq1}
G(\mathbf{x}) = \left\{ 
\begin{aligned}
 &\text{correct}~~~~~~~~~~\text{if} ~~\kappa(\mathbf{x}) \ge \tau, \\
 &\text{misclassified}~~~~\text{otherwise}.
\end{aligned}
\right.
\end{equation} 

For example, we can easily use MSP~\citep{hendrycks2016baseline} as the confidence-scoring function for a given input $\mathbf{x}$ as $\kappa(\mathbf{x}) = \max_{y \in \mathcal{Y}} \hat{p}$. Similarly, other confidence-scoring functions can be adapted from the OOD detection literature, such as MaxLogit~\citep{hendrycks2019scaling},  Energy~\citep{liu2020energy}, and Entropy~\citep{chan2021entropy}.

\subsection{Confidence Degradation: A Failure Indicator in Multimodal Systems}  \label{subsec:when-fail}

We begin by investigating the relationship between multimodal and unimodal prediction confidences to identify systematic patterns that distinguish correct classifications from errors. Our analysis, which uses MSP for confidence scoring, spans four diverse action recognition datasets: HMDB51~\citep{kuehne2011hmdb}, EPIC-Kitchens~\citep{Damen2018EPICKITCHENS}, HAC~\citep{dong2023SimMMDG}, and Kinetics-600~\citep{kay2017kinetics}. We consistently observe a specific failure pattern where the confidence of multimodal prediction $\hat{p}$ falls below that of an individual modality $\hat{p}^k$. We formalize this phenomenon as follows: 

\begin{definition}[Confidence Degradation]
\label{def:confidence-degradation}
A sample is considered to exhibit confidence degradation if the confidence of the fused multimodal prediction is strictly lower than that of at least one of its unimodal counterparts:
\[
\exists\, k \in \{1, \dots, M\} \quad \text{s.t.} \quad \max_{y \in \mathcal{Y}} \hat{p} < \max_{y \in \mathcal{Y}} \hat{p}^k.
\]

\end{definition}




\begin{figure}[t]
  \centering  
\includegraphics[width=\linewidth]{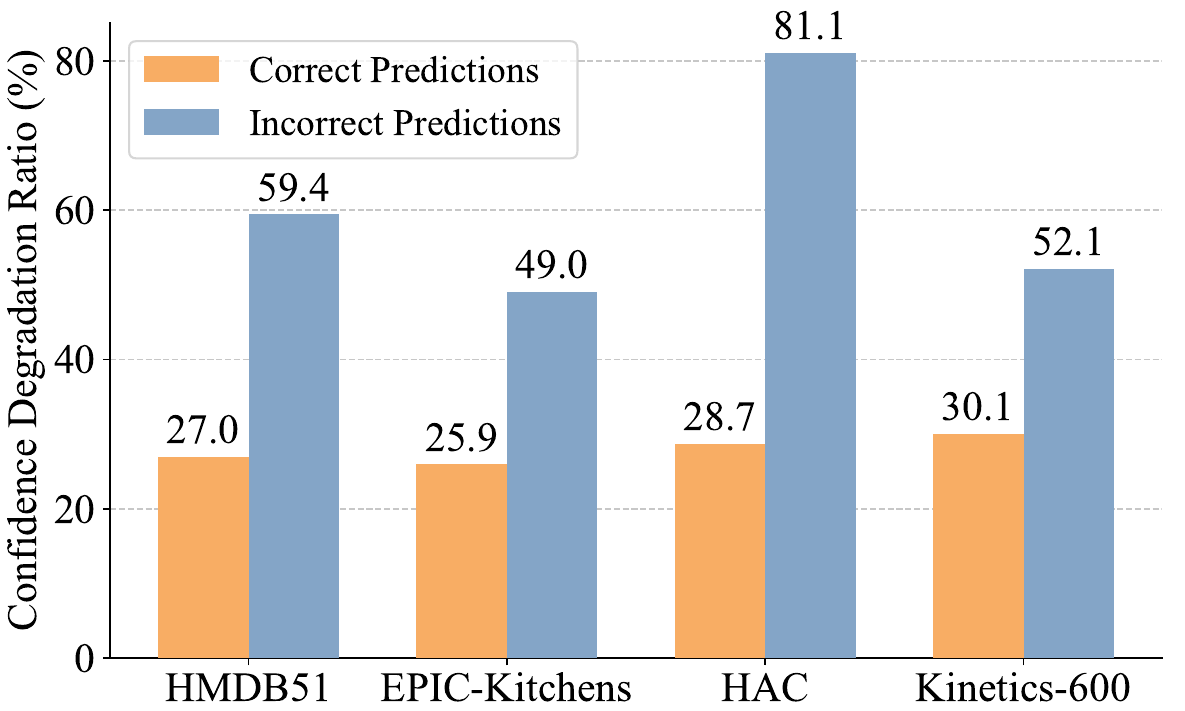}
\vspace{-0.7cm}
   \caption{Misclassified samples exhibit a significantly higher proportion of confidence degradation compared to correctly classified ones. 
}
   \label{fig:motivation}
\vspace{-0.4cm}
\end{figure}

Figure~\ref{fig:motivation} illustrates the central finding: \textit{confidence degradation is strongly associated with prediction failures}. Across all datasets, misclassified samples consistently exhibit a markedly higher rate of degradation than correct predictions, with increases of $32.4\%$ on HMDB51, $23.1\%$ on EPIC-Kitchens, $52.4\%$ on HAC, and $22.0\%$ on Kinetics-600. This suggests that \textit{\textbf{failures in multimodal systems frequently coincide with such confidence degradation}}. One explanation is that misclassified samples often contain conflicting or ambiguous signals across modalities, which increases uncertainty. When their unimodal outputs are fused, this uncertainty frequently causes the combined confidence to drop below that of at least one unimodal branch. In contrast, correctly classified samples typically exhibit agreement across modalities, leading to boosted or at least non‐degraded fusion confidence. 
This directly motivates our adaptive training objective, which explicitly penalizes confidence degradation. 

\subsection{Proposed ACR Framework}

We introduce Adaptive Confidence Regularization (ACR), a novel framework for multimodal failure detection that integrates two complementary components (Figure~\ref{fig:method}). First, motivated by the strong correlation between misclassification and confidence degradation, we propose an Adaptive Confidence Loss that directly penalizes this degradation during training. Second, we introduce Multimodal Feature Swapping, an outlier synthesis technique that generates challenging, failure-aware training samples by exchanging cross-modal embeddings. By training on these synthesized failures, ACR learns a more robust uncertainty representation, improving its ability to reject unreliable predictions.


The ACR architecture processes inputs from multiple modalities. Each input is passed through a modality-specific encoder to yield an embedding, e.g., $\mathbf{E}^{1}$ and $\mathbf{E}^{2}$ for modalities $1$ and $2$. These embeddings are then concatenated, $\mathbf{E} = [\mathbf{E}^{1}, \mathbf{E}^{2}]$, and fed into a fusion classifier to produce the final multimodal prediction $\hat{p}$ with confidence $\textit{conf} = \max_{y \in \mathcal{Y}} \hat{p}$. In parallel, each unimodal embedding $\mathbf{E}^{k}$ is also passed through a dedicated classifier to obtain the unimodal prediction $\hat{p}^k$ and its confidence $\textit{conf}_k$.




\begin{figure*}[t]
  \centering
  \includegraphics[width=0.75\linewidth]{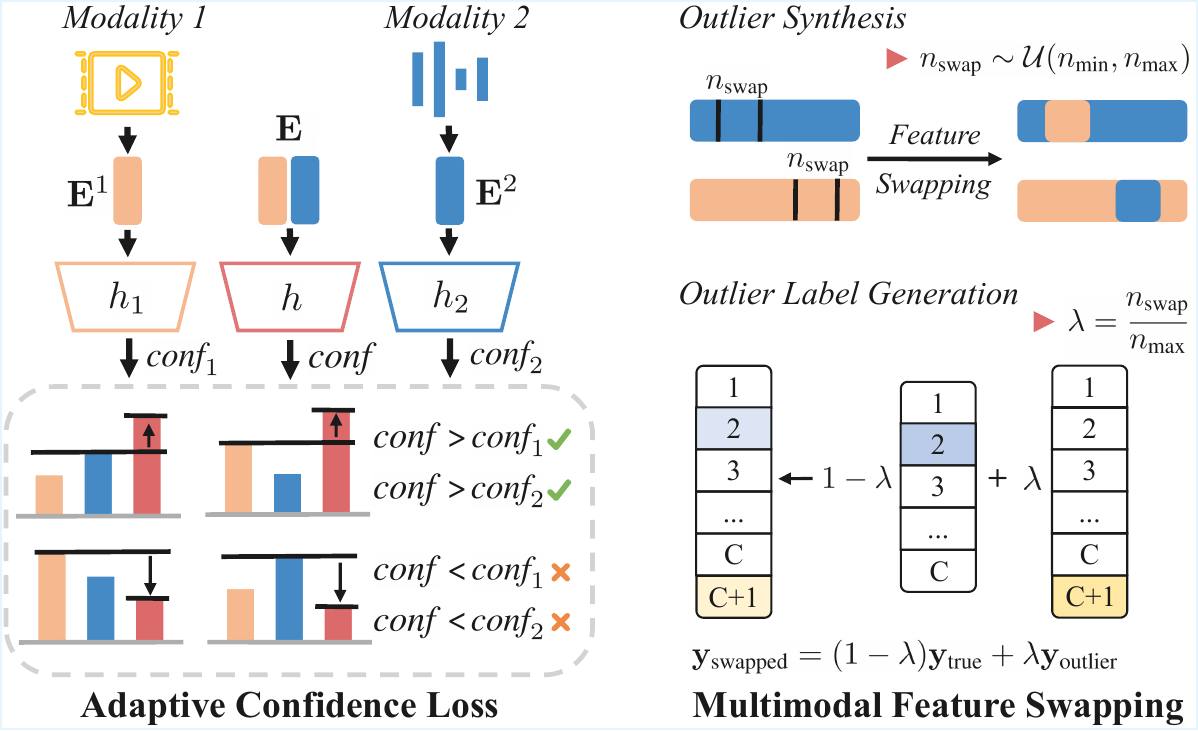}
   \vspace{-0.3cm}
   \caption{
  Our ACR framework integrates two principal components. The Adaptive Confidence Loss is designed to penalize the phenomenon of confidence degradation. The Multimodal Feature Swapping serves to generate challenging, failure-aware training instances. This process enables the model to learn to more effectively identify and reject uncertain samples.
}
   \label{fig:method}
\end{figure*}

\subsection{Adaptive Confidence Loss}\label{subsec:acl}

Ideally, \textbf{\textit{effective multimodal fusion should achieve synergy}}, where the confidence of a fused prediction surpasses that of any single modality, assuming all modalities provide predictive information for the target~\citep{wu2022characterizing}. This reflects the successful integration of complementary information to reduce uncertainty and reinforce the decision. However, as we observe in Section~\ref{subsec:when-fail}, misclassifications are strongly correlated with confidence degradation, a phenomenon where the fused confidence falls below that of a unimodal counterpart. Such degradation often arises from conflicting or unreliable signals and serves as a strong indicator of prediction failure. Motivated by this observation, we introduce the \emph{Adaptive Confidence Loss} (ACL), which encourages the fused confidence to be at least as high as that of any individual modality. 
For a two-modality case, ACL is defined as:
\begin{equation}
\mathcal{L}_{\text{acl}} = \frac{1}{2} \left( \max(0, \textit{conf}_1 - \textit{conf}) + \max(0, \textit{conf}_2 - \textit{conf}) \right).
\end{equation}
The ACL imposes no penalty when the fused confidence surpasses both unimodal confidences; however, it increasingly penalizes instances where the fused confidence is lower than that of either individual modality. 
Consequently, ACL encourages the fusion mechanism to \textit{learn improved information integration}, such that combined evidence from different modalities leads to a more confident prediction. By effectively integrating complementary information from different modalities, \textbf{\textit{ACL enhances prediction reliability}}. 
Furthermore, \textbf{\textit{ACL mitigates unimodal overconfidence}} by penalizing the model when a high-confidence prediction from one modality conflicts with another. To minimize this cross-modal penalty during training, the model learns to reduce the confidence of the unreliable unimodal stream itself. This process effectively regularizes the unimodal networks, forcing them to become better calibrated and less prone to being "confidently wrong". As a result, the model can integrate information more effectively and produce more reliable multimodal predictions. Additional discussion on ACL is provided in the Appendix.

\subsection{Multimodal Feature Swapping}

\begin{figure*}[t]
  \centering
  \includegraphics[width=\linewidth]{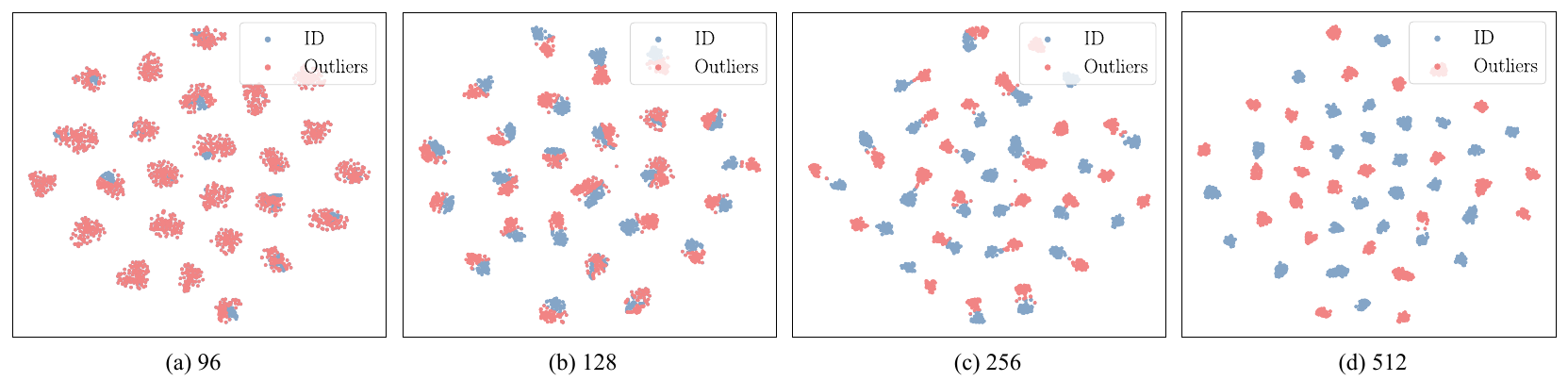}
  \vspace{-0.7cm}
  \caption{Visualization on outliers generated by Multimodal Feature Swapping with different $n_{\text{swap}}$ (96, 128, 256, 512). Small swaps produce hard negatives that lie near the in-distribution manifold, while larger swaps create more distinct outliers further away.}
  \label{fig:mfs}
\end{figure*}


While Outlier Exposure (OE) is an effective technique for improving OOD detection~\citep{hendrycks2018deep,zhang2023mixture}, it has been shown to be ineffective for FD~\citep{OpenMix}. This is because OE regularizes the decision boundary by compressing the confidence distribution of in-distribution (ID) samples, which inadvertently makes it harder to distinguish correct ID predictions from incorrect ones. A related challenge, particularly in multimodal settings, is the lack of training data that realistically emulates system failures, such as conflicting modality cues or sensor corruption. Although approaches like OpenMix~\citep{OpenMix} attempt to address these issues by interpolating between ID and outlier data, they have two critical shortcomings for multimodal tasks. First, they depend on large, auxiliary outlier datasets that are often impractical or unavailable. Second, as a fundamentally unimodal method, OpenMix cannot synthesize the complex failure modes that arise from cross-modal interactions.




To generate challenging, failure-aware outliers without external data, we propose \emph{Multimodal Feature Swapping} (MFS). MFS operates by dynamically swapping multimodal feature embeddings and assigning them corresponding soft labels (as illustrated in Figure~\ref{fig:method}). By generating outliers directly in feature space, MFS ensures computational efficiency and compatibility with various modalities. MFS is designed to ensure that the synthesized features remain distinct from ID features while preserving semantic consistency. Given ID features $\mathbf{E} = [\mathbf{E}^{1}, \mathbf{E}^{2}]$, where $\mathbf{E}^{1}$ represents features from modality~$1$ and $\mathbf{E}^{2}$ from modality $2$, MFS randomly selects a subset of $n_{\text{swap}} \sim \mathcal{U}(n_{\text{min}}, n_{\text{max}})$ continuous feature dimensions from each modality. These selected dimensions are then swapped to obtain new feature representations $\widetilde{\mathbf{E}}^{1}$ and $\widetilde{\mathbf{E}}^{2}$, which are subsequently concatenated to form the multimodal outlier features $\mathbf{E}_o = [\widetilde{\mathbf{E}}^{1}, \widetilde{\mathbf{E}}^{2}]$. A prediction $\hat{p}_o$ is then obtained from $\mathbf{E}_o$ as $\hat{p}_o = \delta(h([\mathbf{E}_o))$. 
To supervise these synthesized outliers, we generate soft labels by interpolating between the original ground-truth one-hot label $\mathbf{y}_{\text{true}}$ and an additional class designated for outliers (e.g., $\mathbf{y}_{\text{outlier}} = C+1$). The weight $\lambda$ for this label interpolation reflects the proportion of features swapped:
\begin{equation}
\mathbf{y}_{\text{swapped}} = (1-\lambda) \mathbf{y}_{\text{true}} + \lambda \mathbf{y}_{\text{outlier}}, \quad \text{where} \quad \lambda = \frac{n_{\text{swap}}}{n_{\text{max}}}.
\end{equation}

MFS generates failure-aware outliers by partially swapping cross-modal features. Such swapping preserves intra-modality semantics while disrupting cross-modal consistency, capturing \textit{a critical and common failure mode} in multimodal systems.
Figure~\ref{fig:mfs} illustrates a t-SNE visualization of the embedding space under different $n_{\text{swap}}$ values. For small $n_{\text{swap}}$, the generated outliers (red) lie close to the ID clusters (blue), acting as hard negatives. As $n_{\text{swap}}$ increases, the outliers gradually move farther from the ID manifold, confirming that MFS provides a controllable mechanism for generating diverse and realistic failure cases. This property is particularly valuable for training models that must remain sensitive to subtle misclassification signals, especially in multimodal scenarios where errors often stem from partial or conflicting evidence. By introducing corrupted or ambiguous multimodal outliers, MFS encodes the prior knowledge of \textbf{\textit{what is uncertain and should be assigned low confidence}}, thereby teaching the model to recognize broader patterns of uncertainty and enhancing its robustness in detecting real-world misclassifications. Additional discussion on MFS is provided in the Appendix.

Overall, MFS offers a simple, generalizable, and computationally efficient approach to simulating realistic failure cases for multimodal failure detection without requiring external data.
The loss for the synthetic outliers is defined as:
\begin{equation}
\mathcal{L}_{\text{outlier}} = \mathrm{CE}(\hat{p}_o, \mathbf{y}_{\text{swapped}}),
\end{equation}
where $\mathrm{CE}$ denotes the cross-entropy loss. The final training objective integrates all components:
\begin{equation}
\mathcal{L}_{\text{total}} = \mathcal{L}_{\text{cls}} + \mathcal{L}_{\text{outlier}} + \lambda_{\text{acl}} \mathcal{L}_{\text{acl}},
\end{equation}
where $\mathcal{L}_{\text{cls}}$ is the cross-entropy loss for the original training samples, and $\lambda_{\text{acl}}$ is a hyperparameter that balances the influence of $\mathcal{L}_{\text{acl}}$.



\subsection{Inference}
Our method focuses on detecting misclassified samples within known classes. Therefore, during the test phase, evaluation is performed exclusively on the original $C$ classes. Specifically, for a given input $\mathbf{x}$, the predicted label is $\hat{y} = \mathop{\mathrm{argmax}}_{y \in \mathcal{Y}} \hat{p}$, and the corresponding confidence is determined using the common MSP score, \emph{i.e.}, $\kappa(\mathbf{x}) = \max_{y \in \mathcal{Y}} \hat{p}$. 

\section{Experiments}

\subsection{Experimental Setup}
\label{subsec:Datasets}
\noindent\textbf{Datasets.}  We evaluate our proposed framework on four action recognition datasets sourced from the MultiOOD benchmark~\citep{dong2024multiood}: HMDB51~\citep{kuehne2011hmdb}, Kinetics-600~\citep{kay2017kinetics}, HAC~\citep{dong2023SimMMDG}, and EPIC-Kitchens~\citep{Damen2018EPICKITCHENS}. Each of these datasets incorporates video and optical flow modalities. For the HAC dataset, we also include evaluations utilizing the audio modality. 
Further details on each dataset are in the Appendix.


\noindent\textbf{Implementation.}
We conduct experiments across three modalities: video, audio, and optical flow. The MMAction2~\citep{2020mmaction2} toolkit is adopted for all experiments. To encode visual information, we utilize the SlowFast network~\citep{Feichtenhofer_2019_ICCV}, initialized with weights pre-trained on the Kinetics-400 dataset~\citep{kay2017kinetics}. For the audio encoder, we employ a ResNet-18 architecture~\citep{he2016deep}, with weights initialized from the VGGSound pre-trained checkpoint~\citep{9053174}. Similarly, the optical flow encoder uses the SlowFast network, configured with a slow-only pathway and also leveraging pre-trained weights from Kinetics-400~\citep{kay2017kinetics}.
The Adam optimizer~\citep{Adam} is used for model training, with a learning rate of $0.0001$ and a batch size of $16$. The hyperparameters for our proposed method are set as follows: $\lambda_{\text{acl}} = 2.0$, $n_{\text{min}} = 32$, $n_{\text{max}} = 256$. We train the models for $50$ epochs on an NVIDIA RTX 3090 GPU and select the model with the best performance on the validation dataset. 



\noindent\textbf{Baselines.}
We compare our approach against several standard confidence-scoring functions, including MSP~\citep{hendrycks2016baseline}, MaxLogit~\citep{hendrycks2019scaling}, Energy~\citep{liu2020energy}, and Entropy~\citep{chan2021entropy}. Additionally, we adapt unimodal FD methods for our framework, including DOCTOR~\citep{granese2021doctor} and OpenMix~\citep{OpenMix}, 
along with the outlier synthesis techniques Mixup~\citep{zhang2017mixup}, RegMixup~\citep{pinto2022using}. We also include established training strategies, namely CRL~\citep{moon2020confidence} and A2D~\citep{dong2024multiood}, where A2D is designed for multimodal OOD detection. These baselines collectively represent a diverse array of techniques for FD. 


\setlength{\tabcolsep}{1.5pt}

\begin{table*}[t!]
\centering
\resizebox{0.95\textwidth}{!}{
\begin{tabular}{l@{~~~~~}| c cccc cccc cccc cccc }
\hline

 \multicolumn{1}{c|}{{} } &\multicolumn{4}{c|}{HMDB51} & \multicolumn{4}{c|}{EPIC-Kitchens} & \multicolumn{4}{c|}{HAC} &\multicolumn{4}{c}{Kinetics-600} 
 \\
   \multicolumn{1}{c|}{{}} & AURC$\downarrow$ & AUROC$\uparrow$ & FPR95$\downarrow$ &  \multicolumn{1}{c|}{ACC$\uparrow$} &  AURC$\downarrow$ & AUROC$\uparrow$ & FPR95$\downarrow$ &  \multicolumn{1}{c|}{ACC$\uparrow$}  & AURC$\downarrow$ & AUROC$\uparrow$ & FPR95$\downarrow$ &  \multicolumn{1}{c|}{ACC$\uparrow$} & AURC$\downarrow$ & AUROC$\uparrow$ & FPR95$\downarrow$ &  \multicolumn{1}{c}{ACC$\uparrow$}\\
\hline

 
 

 \multicolumn{1}{c|}{MaxLogit}&34.76 &85.65  & 57.02 & \multicolumn{1}{c|}{86.20}  & 114.22   &76.92  & 80.00 & \multicolumn{1}{c|}{74.25}  & 53.61& 85.12 &58.97  & \multicolumn{1}{c|}{82.11}  &  63.90&81.25  & 69.85 & \multicolumn{1}{c}{81.24}  \\

 \multicolumn{1}{c|}{Energy}&35.78 &85.03  &58.68  & \multicolumn{1}{c|}{86.20}  &  114.91  &76.67  &78.95  & \multicolumn{1}{c|}{74.25}  &  54.33& 84.80  &58.97  & \multicolumn{1}{c|}{82.11}  &  66.48& 80.12 &76.58  & \multicolumn{1}{c}{81.24}  \\
 
 \multicolumn{1}{c|}{Entropy}& 30.24&  87.87&  57.85& \multicolumn{1}{c|}{86.20}  & 114.09 &  76.83& 78.42 & \multicolumn{1}{c|}{74.25}  & 42.63& 89.46 &61.54  & \multicolumn{1}{c|}{82.11}  &  47.30&86.85  & 65.22 & \multicolumn{1}{c}{81.24}  \\
 
 \multicolumn{1}{c|}{MSP}&29.56 &88.28  &52.07  & \multicolumn{1}{c|}{86.20}  &  115.03  & 76.52 & 76.84 & \multicolumn{1}{c|}{74.25}  & 42.90& 89.27  &66.67  & \multicolumn{1}{c|}{82.11}  &46.29  &87.33  & 61.29 & \multicolumn{1}{c}{81.24}  \\

 \multicolumn{1}{c|}{DOCTOR}&29.65 &88.42  &52.46  & \multicolumn{1}{c|}{86.20}  & 114.92   &76.57  &79.47  & \multicolumn{1}{c|}{74.25}  & 42.60& 89.46  &64.10  & \multicolumn{1}{c|}{82.11}  & 46.37 &87.28  &62.27  & \multicolumn{1}{c}{81.24}  \\
 
 \multicolumn{1}{c|}{Mixup}&36.52 & 87.98 &50.00  & \multicolumn{1}{c|}{84.72}  &  110.54  &77.72  & 75.41 & \multicolumn{1}{c|}{75.20}  &34.06 & 87.88 & 55.88 & \multicolumn{1}{c|}{84.40}  & 50.56 & 86.87 &60.57  & \multicolumn{1}{c}{80.58}  \\

 \multicolumn{1}{c|}{RegMixup}&29.86 & 88.25 & 55.37 & \multicolumn{1}{c|}{86.20}  &  105.25  & 79.26 &78.19  & \multicolumn{1}{c|}{74.53}  & 50.28& 82.83 & 72.22 & \multicolumn{1}{c|}{83.49}  &51.44  &86.06  &62.71  & \multicolumn{1}{c}{81.16}  \\

 \multicolumn{1}{c|}{OpenMix}& 24.15& 90.13 & 51.33 & \multicolumn{1}{c|}{87.12}  &    112.14& 78.46 &73.68  & \multicolumn{1}{c|}{74.25}  & 35.42& 87.51 & 54.84 & \multicolumn{1}{c|}{83.03}  &46.73  &87.69  &60.27  & \multicolumn{1}{c}{80.79}  \\
 

 \multicolumn{1}{c|}{A2D}& 25.01& 89.79 & 47.01 & \multicolumn{1}{c|}{86.66}  & 109.90   &77.85  & 76.72 & \multicolumn{1}{c|}{74.39}  &45.89 &89.77  &57.14  & \multicolumn{1}{c|}{83.94}  &49.26  &87.97  &59.79  & \multicolumn{1}{c}{79.97}  \\

 \multicolumn{1}{c|}{CRL}& 26.44&90.39  &46.40  & \multicolumn{1}{c|}{85.75}  &  107.42  &78.33  &79.17  & \multicolumn{1}{c|}{73.98}  &36.46 & 86.53 & 59.38 & \multicolumn{1}{c|}{83.49}  &  49.16& 87.29 &61.73  & \multicolumn{1}{c}{80.47}  \\


 \rowcolor{gray!20}
 \multicolumn{1}{c|}{ACR (ours)}& \textbf{19.97}& \textbf{92.02} &\textbf{41.96}  & \multicolumn{1}{c|}{\textbf{87.23}}  &\textbf{103.25}    &\textbf{79.27}  &\textbf{71.58}  & \multicolumn{1}{c|}{\textbf{75.20}}  &\textbf{27.41} &\textbf{91.48}  &\textbf{39.39}  & \multicolumn{1}{c|}{\textbf{84.86}}  &\textbf{41.85}  &\textbf{88.99}  &\textbf{55.89}  & \multicolumn{1}{c}{\textbf{81.45}}  \\

 \hline

\end{tabular} 
}
\vspace{-0.2cm}
\caption{FD performance on action recognition datasets with video and optical flow modalities.}
\vspace{-0.3cm}
\label{tab:FD1} 
\end{table*}

\noindent\textbf{Evaluation Metrics.} Following~\citep{OpenMix}, we assess FD performance using the following metrics:
(1) \textbf{AURC} (Area Under the Risk-Coverage Curve), which measures the model's risk (error rate) as a function of coverage (fraction of samples retained). The AURC value is multiplied by $10^3$ following~\citep{OpenMix}.
(2) \textbf{AUROC} (Area Under the Receiver Operating Characteristic Curve), quantifying the trade-off between the true positive rate (TPR) and the false positive rate (FPR);
(3) \textbf{FPR95} (False Positive Rate at 95\% TPR), indicating the proportion of incorrectly classified samples that are misidentified as correct when the TPR is fixed at $95\%$;
(4) \textbf{ACC} (Accuracy), representing the standard test accuracy on the ID data.

\subsection{Main Results} \label{subsec:main_results}

\noindent\textbf{Performance on Multimodal FD.}
Table~\ref{tab:FD1} presents a comparative analysis of our method against various baseline approaches across four datasets—HMDB51, EPIC-Kitchens, HAC, and Kinetics-600—utilizing video and optical flow as input modalities. Our proposed method consistently outperforms all baselines across key FD metrics. For instance, on the HMDB51 dataset, our method reduces the FPR95 from $52.07\%$ to $41.96\%$ and improves the AUROC from $88.28\%$ to $92.02\%$ when compared to the MSP baseline. On the HAC dataset, it achieves a reduction in AURC from $42.90\%$ to $27.41\%$ and an increase in AUROC from $89.27\%$ to $91.48\%$. Furthermore, on the Kinetics-600 dataset, our method reduces the FPR95 from $61.29\%$ to $55.89\%$. In addition to these FD improvements, our method even improves classification accuracy on all evaluated datasets. These consistent gains observed across diverse video domains underscore the strong generalization capabilities and robustness of the proposed method. 


\noindent\textbf{Performance under Different Modality Combinations.}  
\Cref{tab:FD-vfa} presents results comparing our method against baseline approaches on the HAC dataset under various modality combinations: video+audio, optical flow+audio, and video+optical flow+audio. While \Cref{tab:FD1} exclusively reports results using video and optical flow, this evaluation investigates the generalization of our method across different modality configurations. Our method surpasses baselines in most scenarios, achieving average improvements of $8.39\%$ in AURC, $1.51\%$ in AUROC, and $10.65\%$ in FPR95 relative to the strongest baseline. Concurrently, it enhances classification accuracy from $81.19\%$ to $82.42\%$. These improvements demonstrate the robustness of our approach to diverse modality configurations and its efficacy in enhancing both accuracy and FD performance. 

\begin{figure}
  \centering  
\includegraphics[width=\linewidth]{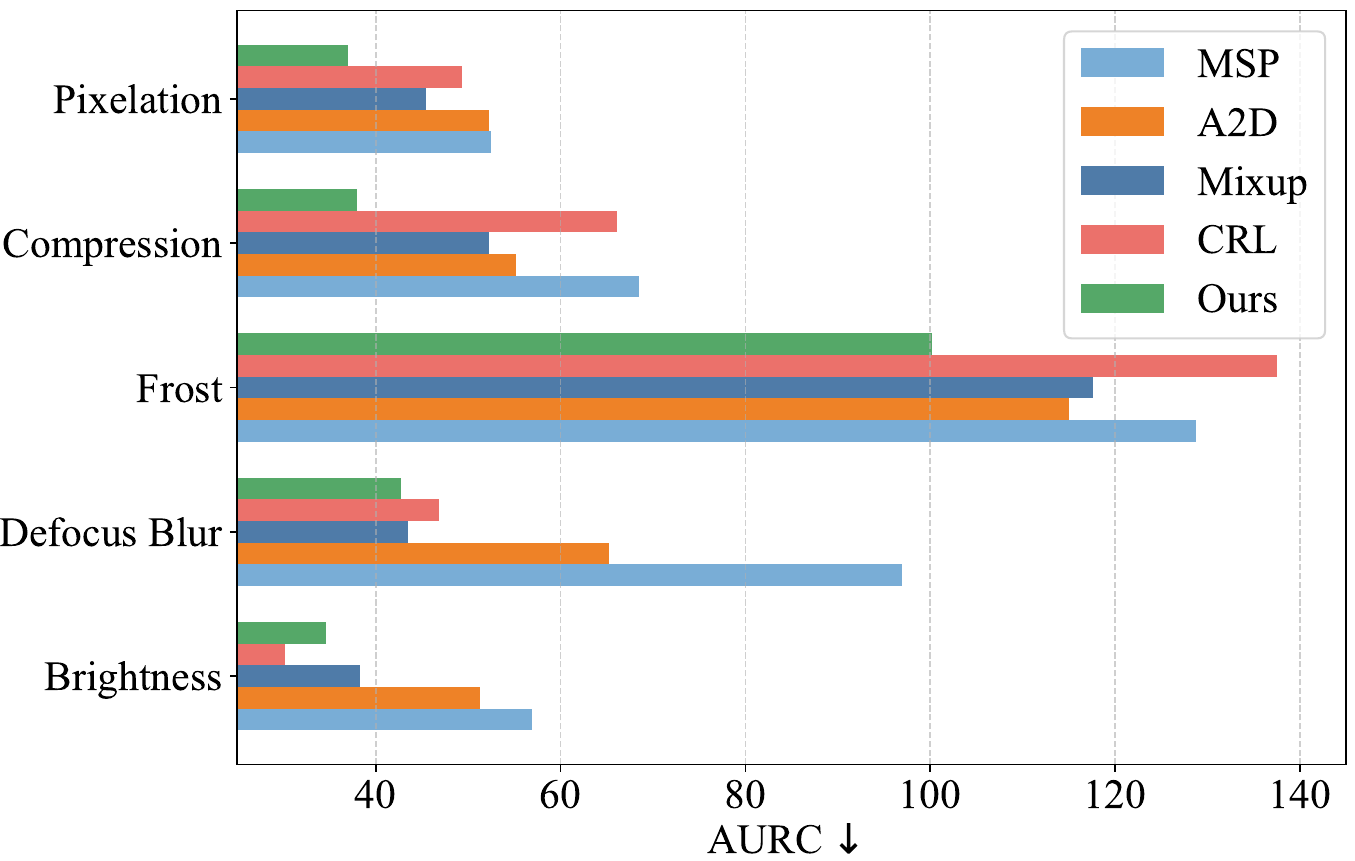}
\vspace{-0.7cm}
   \caption{FD under distribution shift on HAC dataset. The performance of five types of corruption on videos under the severity level of $5$ is reported.
   }
   \label{fig:corrupt}
\vspace{-0.3cm}
\end{figure}

\begin{table}[t!]
\centering
\resizebox{0.7\linewidth}{!}{
\begin{tabular}{l | c c c c}
\hline
\multicolumn{1}{c|}{} & AURC$\downarrow$ & AUROC$\uparrow$ & FPR95$\downarrow$ & ACC$\uparrow$ \\
\hline

MSP         & 29.56& 88.28&52.07  & \multicolumn{1}{c}{86.20} \\

ACL         & 24.48& 90.32&43.97  & \multicolumn{1}{c}{86.77} \\

MFS        & 25.11& 90.55 &46.22  & \multicolumn{1}{c}{86.43}  \\

\rowcolor{gray!20} 
ACL + MFS & \textbf{19.97}& \textbf{92.02} &\textbf{41.96}  & \multicolumn{1}{c}{\textbf{87.23}}\\

\hline
\end{tabular}
}
\vspace{-0.2cm}
\caption{Effect of each component.}
\vspace{-0.3cm}
\label{tab:FD-comp}
\end{table}

\setlength{\tabcolsep}{1.5pt}

\begin{table*}[t!]
\centering
\resizebox{0.95\textwidth}{!}{
\begin{tabular}{l@{~~~~~}| c cccc cccc cccc cccc }
\hline

 \multicolumn{1}{c|}{{} } &\multicolumn{4}{c|}{video+audio} & \multicolumn{4}{c|}{optical flow+audio} & \multicolumn{4}{c|}{video+optical flow+audio} &\multicolumn{4}{c}{Average} 
 \\
   \multicolumn{1}{c|}{{}} & AURC$\downarrow$ & AUROC$\uparrow$ & FPR95$\downarrow$ &  \multicolumn{1}{c|}{ACC$\uparrow$} &  AURC$\downarrow$ & AUROC$\uparrow$ & FPR95$\downarrow$ &  \multicolumn{1}{c|}{ACC$\uparrow$}  & AURC$\downarrow$ & AUROC$\uparrow$ & FPR95$\downarrow$ &  \multicolumn{1}{c|}{ACC$\uparrow$} & AURC$\downarrow$ & AUROC$\uparrow$ & FPR95$\downarrow$ &  \multicolumn{1}{c}{ACC$\uparrow$}\\
\hline

 \multicolumn{1}{c|}{MaxLogit}& 28.35& 84.41 &68.00  & \multicolumn{1}{c|}{88.07}  &   117.73 & 83.08 & 72.46 & \multicolumn{1}{c|}{68.35}  &26.34 &87.82  &75.00  & \multicolumn{1}{c|}{87.16}  &57.47 & 85.10 & 71.82 & \multicolumn{1}{c}{81.19}  \\

 \multicolumn{1}{c|}{Energy}& 29.32&  83.90& 68.00 & \multicolumn{1}{c|}{88.07}  & 120.43   &82.37  &75.36  & \multicolumn{1}{c|}{68.35}  &27.63 & 86.99 &75.00  & \multicolumn{1}{c|}{87.16}  & 59.13 & 84.42 & 72.79  & \multicolumn{1}{c}{81.19}  \\
 
 \multicolumn{1}{c|}{Entropy}& 25.10&  86.03& 64.00 & \multicolumn{1}{c|}{88.07}  &    116.95& 82.75 &66.67  & \multicolumn{1}{c|}{68.35}  &21.40 &91.00  & 57.14 & \multicolumn{1}{c|}{87.16}  & 54.48 & 86.59 & 62.60& \multicolumn{1}{c}{81.19}  \\
 
 \multicolumn{1}{c|}{MSP}& 26.79& 86.30& 57.69& \multicolumn{1}{c|}{88.07}  &  118.04  & 82.46 & 71.01 & \multicolumn{1}{c|}{68.35}  & 20.99&91.35  &42.86  & \multicolumn{1}{c|}{87.16} & 55.27 & 86.70 & 57.19 & \multicolumn{1}{c}{81.19}  \\

 \multicolumn{1}{c|}{DOCTOR}&27.10 & 86.02 & 53.85 & \multicolumn{1}{c|}{\textbf{88.07}}  & 118.03   & 82.51 & 72.46 & \multicolumn{1}{c|}{68.35}  & 21.03& 91.33 &46.43  & \multicolumn{1}{c|}{87.16}  &55.39 & 86.62 & 57.58& \multicolumn{1}{c}{81.19}  \\
 

 \multicolumn{1}{c|}{Mixup}&26.03 &88.72 &54.84  & \multicolumn{1}{c|}{85.78}  &   147.96 &76.69  &78.57  & \multicolumn{1}{c|}{67.89}  & 21.76& 86.96 & 69.57 & \multicolumn{1}{c|}{89.45}  &65.25 & 84.12 & 67.66 & \multicolumn{1}{c}{81.04}  \\
 
 \multicolumn{1}{c|}{RegMixup}&25.58 & 89.56 &73.33  & \multicolumn{1}{c|}{86.24}  &    134.26& 78.00 &86.96  & \multicolumn{1}{c|}{68.35}  & 29.77& 86.43 &68.97  & \multicolumn{1}{c|}{86.71}  &63.20 & 84.66 & 76.42 & \multicolumn{1}{c}{80.43}  \\

 \multicolumn{1}{c|}{OpenMix}&26.09 & 88.33 &64.29  & \multicolumn{1}{c|}{87.16}  & 117.48   & 81.99 &71.88  & \multicolumn{1}{c|}{70.64}  & 20.10&91.04  &48.15  & \multicolumn{1}{c|}{87.61}  & 54.56 & 87.12 & 61.44 & \multicolumn{1}{c}{81.80}  \\
 

 \multicolumn{1}{c|}{A2D}&59.65 &81.36  & 60.53 & \multicolumn{1}{c|}{82.57}  &  108.07  &81.27  & 79.37 & \multicolumn{1}{c|}{71.10}  & 20.16& 90.81 &48.28  & \multicolumn{1}{c|}{86.70}  &62.63 & 84.48 & 62.73& \multicolumn{1}{c}{80.12}   \\

 \multicolumn{1}{c|}{CRL}& 26.88& 89.29 & 43.33 & \multicolumn{1}{c|}{86.24}  & 118.65   &83.47  & 71.83 & \multicolumn{1}{c|}{67.43}  & 31.16& 87.42 &64.52  & \multicolumn{1}{c|}{85.78}  &58.90 & 86.73 & 59.89& \multicolumn{1}{c}{79.82}  \\

 \rowcolor{gray!20}
 \multicolumn{1}{c|}{ACR (ours)}&\textbf{24.70}& \textbf{89.67} & \textbf{41.38} & \multicolumn{1}{c|}{86.70} &\textbf{98.49} &\textbf{83.96}  &\textbf{63.47} & \multicolumn{1}{c|}{\textbf{71.10}}&\textbf{15.09}&\textbf{92.26}  &\textbf{34.78}  & \multicolumn{1}{c|}{\textbf{89.45}}  &\textbf{46.09} & \textbf{88.63} & \textbf{46.54} & \multicolumn{1}{c}{\textbf{82.42}}  \\

 \hline

\end{tabular} 
}
\vspace{-0.2cm}
\caption{FD performance on HAC dataset with different modality combinations. }
\vspace{-0.3cm}
\label{tab:FD-vfa} 
\end{table*}

\noindent\textbf{Performances under Distribution Shifts.} In practical applications, environmental conditions can change rapidly (e.g., weather transitioning from sunny to cloudy, then to rainy), necessitating reliable model decisions under such distribution or domain shifts. To simulate these scenarios, we evaluate model performance under various data corruptions with a severity level of $5$, including Defocus Blur, Frost, Brightness, Pixelate, and JPEG Compression. Models were trained on the clean HAC dataset using video and optical flow modalities, with corruptions introduced to the video modality exclusively during the testing phase. As illustrated in Figure \ref{fig:corrupt}, our framework demonstrates significantly improved FD performance on AURC under various corruptions in the majority of tested cases.





\subsection{Ablation Studies}

\noindent\textbf{Effect of Each Component.}  Table \ref{tab:FD-comp} summarizes performance gains contributed by each component 
of our framework, evaluated on the HMDB dataset. Commencing from the MSP baseline, the individual incorporation of either the Adaptive Confidence Loss or the Multimodal Feature Swapping module leads to performance improvements across all metrics. Crucially, the combination of both components yields the most substantial overall enhancements. These findings underscore the complementary strengths of the two proposed modules. 

\noindent\textbf{Robustness to Different Architectures.}  
To demonstrate the scalability of our approach, we evaluate 
\setlength{\tabcolsep}{2pt}
\begin{table}[t!]
\centering
\resizebox{0.8\linewidth}{!}{
\begin{tabular}{l | c c c c}
\hline
\multicolumn{1}{c|}{} & AURC$\downarrow$ & AUROC$\uparrow$ & FPR95$\downarrow$ & Accuracy$\uparrow$ \\
\hline
MSP         & 31.77 & 88.17 & 58.59 & 85.40 \\
DOCTOR      & 31.76 & 88.18 & 59.38 & 85.40 \\
Mixup       & 41.12 & 87.29 & 56.76 & 83.12 \\
RegMixup    & 41.67 & 89.41 & 54.27 & 81.30 \\
OpenMix      & 33.16 & 88.30 & 55.12 & 84.26 \\
A2D         & 34.47 & 88.38 & 55.80& 84.72 \\
CRL         & 37.13 & 89.09 & 60.78 & 82.55 \\

\rowcolor{gray!20} 
ACR (ours) & \textbf{27.73} & \textbf{90.00} & \textbf{51.56} & \textbf{85.40} \\
\hline
\end{tabular}
}
\vspace{-0.2cm}
\caption{Ablation on different architectures.}
\vspace{-0.3cm}
\label{tab:FD-i3d}
\end{table}

its performance using alternative backbone encoders, specifically, 
I3D~\citep{carreira2017quo} and TSN~\citep{wang2016temporal}, for extracting video and optical flow features. 
The results on the HMDB51 dataset are presented in Table~\ref{tab:FD-i3d}. Despite the utilization of lighter and structurally distinct architectures, our method maintains competitive performance across all four evaluation metrics. It consistently achieves lower FPR95 and AURC, and higher AUROC values compared to all baseline methods. 

\noindent\textbf{ACR Improves Multimodal OOD Detection.} A reliable multimodal system should be capable of distinguishing both OOD samples and misclassified ID samples from correct predictions. Therefore, in addition to FD, we investigate the OOD detection capabilities of our method using the MultiOOD benchmark~\citep{dong2024multiood}, focusing on video and optical flow modalities. The ID dataset employed is HMDB51. For OOD datasets, we utilize Kinetics-600, UCF101~\citep{soomro2012ucf101}, EPIC-Kitchens, and HAC. Performance is evaluated using AUROC, FPR95, and ID Accuracy. We train models using both the A2D+NP-Mix (AN)~\citep{dong2024multiood} strategies in MultiOOD and our ACR framework, and subsequently evaluate them with various OOD confidence-scoring functions, including MSP, Energy, MaxLogit, and GEN~\citep{10203747}. The results, presented in Table \ref{tab:hmdb_results_less}, indicate that ACR not only exhibits strong FD capabilities but also achieves robust OOD detection performance in comparison to AN.

\begin{table}[t!]
\centering
\resizebox{0.6\linewidth}{!}{
\begin{tabular}{l |  c c c}
\hline
\multicolumn{1}{c|}{} & AUROC$\uparrow$ & FPR95$\downarrow$ & ACC$\uparrow$ \\
\hline

MSP         & 94.14& 39.35& 86.20 \\

Mixup         & 95.35& 24.92& 84.72\\

A2D        & 94.33& 29.90& 86.66\\

CRL        & 94.14& 32.29& 85.75\\

\rowcolor{gray!20} 
ACR (ours) & \textbf{96.82}& \textbf{20.47}& \textbf{87.23}\\

\hline
\end{tabular}
}
\vspace{-0.2cm}
\caption{Performance under both OOD and misclassified samples.}
\vspace{-0.3cm}
\label{tab:FD-new}
\end{table}

\begin{table*}[t]
  \centering
  \scalebox{0.8}{
    \begin{tabular}{cccccccccccc}
    \toprule
    \multirow{3}[6]{*}{Methods} & \multicolumn{10}{c}{\textbf{OOD Datasets}}                                             & \multirow{3}[6]{*}{ID ACC $\uparrow$ } \\
\cmidrule{2-11}          & \multicolumn{2}{c}{\textbf{Kinetics-600}} & \multicolumn{2}{c}{\textbf{UCF101}} & \multicolumn{2}{c}{\textbf{EPIC-Kitchens}} & \multicolumn{2}{c}{\textbf{HAC}} & \multicolumn{2}{c}{Average} &  \\
\cmidrule{2-11}        & FPR95$\downarrow$ & AUROC$\uparrow$ & FPR95$\downarrow$ & AUROC$\uparrow$ & FPR95$\downarrow$ & AUROC$\uparrow$ & FPR95$\downarrow$ & AUROC$\uparrow$ & FPR95$\downarrow$ & AUROC$\uparrow$ &  \\
    \midrule

    MSP   &   39.11 &  88.78 &  46.64  & 86.40   &  17.33  & 95.99   & 39.91   &  89.10  &  35.75  & 90.07   &   87.23 \\

    ~~~~~~+AN   &  29.42&  90.73&  40.02&  88.08&  13.34&  96.43&  28.16&  91.63&  27.74&  91.72&  86.89    \\

    ~~~~~~+Ours   & 27.71 &   92.81 &   34.89 &   89.10 &   12.20 &   97.76 &   21.78 &   94.25 &\textbf{24.15}    &\textbf{93.48}    &   87.23     \\

    \midrule

    Energy   &  32.95  & 92.48  &   44.93 &   87.95 &  8.10  &  97.70  &  32.95  & 92.28   &  29.73  &   92.60 &  87.23  \\

    ~~~~~~+AN   & 24.52&  93.96&  36.49&  89.67&  6.96&  97.53&  22.92&  94.41&  22.72&  93.89&  86.89     \\

    ~~~~~~+Ours   &18.59   & 95.39 &31.58  &91.05  & 8.32 & 96.57 & 13.45 &96.07  &\textbf{17.99}  &\textbf{94.77}  &  87.23    \\

    \midrule
    
    MaxLogit   &  33.07  & 92.31  &   44.93 & 88.02   &    9.12  & 97.77 &  33.06  &  92.17  & 30.05   &  92.57  &  87.23  \\

    ~~~~~~+AN   &  24.86&  93.69&  36.60&  89.71&  6.96&  97.67&  22.92&  94.22&  22.84&  93.82&  86.89    \\

    ~~~~~~+Ours   & 19.16  &95.27  &31.93  &91.00  &8.21  & 97.75 & 14.82 &96.15  &\textbf{18.53}  &\textbf{95.04} &  87.23    \\

    \midrule
    
    GEN   &   41.51 & 90.34  &   46.18 &  87.91  & 8.21 &  98.26  & 38.31   &  91.28  & 33.55   & 91.95   &  87.23  \\

    ~~~~~~+AN   &   25.66&  93.50&  37.40&  91.19&  5.25&  98.98&  24.63&  94.28&  23.24&  94.49&  86.89   \\

    ~~~~~~+Ours   &  22.46 &95.17  &31.58  &92.19  &2.62  & 99.38 &15.17  & 96.62 &\textbf{17.96}  &\textbf{95.84}  &  87.23    \\

    \bottomrule
    \end{tabular}
}
\vspace{-0.2cm}
  \caption{{Multimodal OOD detection} using video and optical flow, with {HMDB51} as ID.}
\vspace{-0.5cm}
\label{tab:hmdb_results_less}
\end{table*}%

\noindent\textbf{FD in Presence of OOD Samples. } 
In this challenging setting, OOD samples are introduced
into the test data, and the model is required to distinguish both OOD samples and misclassified ID samples from correct predictions. We use HMDB51 as the ID dataset and incorporate OOD samples from the HAC dataset. As shown in \Cref{tab:FD-new}, our ACR demonstrates strong robustness in this scenario, accurately identifying both OOD and misclassified samples. It achieves average improvements of $1.47\%$ in AUROC, $4.45\%$ in FPR95, and $0.57\%$ in ACC compared to the strongest baseline.

\noindent\textbf{Compare with Other Feature-space Augmentation Methods.} To assess the effectiveness of 
\begin{table}[t]
\centering
\resizebox{0.8\linewidth}{!}{
\begin{tabular}{l | c c c c}
\hline
\multicolumn{1}{c|}{} & AURC$\downarrow$ & AUROC$\uparrow$ & FPR95$\downarrow$ & ACC$\uparrow$ \\
\hline

MSP   &  29.56  &  88.28   & 52.07  & \multicolumn{1}{c}{86.20} \\ 
Random Noise   &  24.86& 	90.82& 	42.86	    & \multicolumn{1}{c}{86.43} \\

 Random Drop  &  22.80 & 91.24    & 46.09      & \multicolumn{1}{c}{86.89} \\

 Feature Mixing  &   21.79&  91.33   & 42.11      & \multicolumn{1}{c}{87.00} \\

\rowcolor{gray!20} 
MFS  &  	\textbf{19.97}&	\textbf{92.02}	&\textbf{41.96}	   & \multicolumn{1}{c}{\textbf{87.23}} \\

\hline
\end{tabular}
}
\vspace{-0.2cm}
\caption{Compare with other feature space augmentation methods.}
\vspace{-0.3cm}
\label{tab:aug}
\end{table}

alternative augmentation strategies, we replace MFS with Random Noise (randomly replacing embedding values with noise), Random Drop (randomly replacing embedding values with zeros), and Feature Mixing~\citep{liu2025fm}. As shown in Table~\ref{tab:aug}, all baseline methods improve FD performance, highlighting the importance of outlier synthesis for regularization. However, MFS proves most effective, as it generates outliers of varying difficulty that better capture cross-modal inconsistency.


\noindent\textbf{Evaluation on More Modalities.} To further evaluate the robustness of our framework, we conduct 
\begin{table}[t]
\centering
\resizebox{0.75\linewidth}{!}{
\begin{tabular}{l | c c c c}
\hline
\multicolumn{1}{c|}{} & AURC$\downarrow$ & AUROC$\uparrow$ & FPR95$\downarrow$ & mIoU$\uparrow$ \\
\hline

MSP   & 33.90	 &79.97 &	55.49		    & \multicolumn{1}{c}{59.25} \\

\rowcolor{gray!20} 
ACR (ours)  &  	\textbf{21.90} &	\textbf{84.51}	 &\textbf{52.51}		   & \multicolumn{1}{c}{\textbf{63.56}} \\

\hline
\end{tabular}
}
\vspace{-0.2cm}
\caption{Results on SemanticKITTI for 3D semantic segmentation.}
\label{tab:FD-seg}
\vspace{-0.3cm}
\end{table}

experiments on the SemanticKITTI dataset~\citep{behley2019semantickitti}, using image and LiDAR point cloud modalities for the 3D semantic segmentation task. We adopt the fusion framework of~\cite{pmf} and adapt the evaluation to the FD setting. As shown in Table~\ref{tab:FD-seg}, our framework also achieves strong failure detection performance for the semantic segmentation task.






\noindent\textbf{Visualization.} To qualitatively assess confidence score distributions, we visualized them for correct and incorrect predictions on the HMDB51 dataset, as depicted in Figure \ref{fig:score}. The baseline MSP method exhibits a less distinct separation in confidence scores between correctly classified and misclassified samples. In contrast, our proposed solution assigns higher confidence scores to correct predictions and discernibly lower scores to incorrect ones. This leads to more clearly separated distributions, thereby enhancing the efficacy of failure detection.

\begin{figure}[t]
  \centering
  \includegraphics[width=0.75\linewidth]{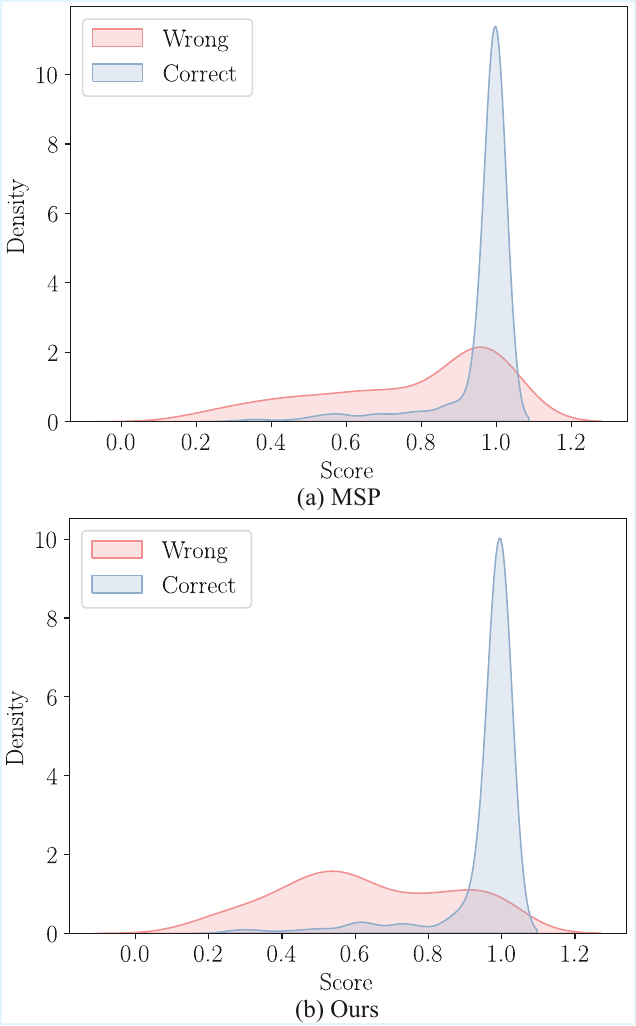}
  \vspace{-0.3cm}
  \caption{Score distribution for correct and wrong predictions. Our method leads to more clearly separated distributions, thereby enhancing the efficacy of failure detection.}
  \vspace{-0.3cm}
  \label{fig:score}
\end{figure}

\section{Conclusion}
In this work, we addressed the critical yet underexplored challenge of failure detection in multimodal systems, a vital component for ensuring reliability in safety-sensitive domains. We introduced ACR, the first dedicated framework designed to tackle this problem. By characterizing the confidence degradation phenomenon—where fused multimodal predictions exhibit lower confidence than their unimodal counterparts in most error cases—we introduced an Adaptive Confidence Loss that directly penalizes this behavior during training. Complementing this loss, our Multimodal Feature Swapping technique synthesizes challenging outliers, further enhancing the model’s ability to flag unreliable predictions.  Extensive evaluations across four diverse datasets and three modalities demonstrated ACR's superior performance and generalization capabilities, significantly outperforming existing baselines. By enabling more reliable identification of untrustworthy predictions, ACR represents a significant step towards enhancing the safety and trustworthiness of multimodal systems in real-world applications. 



\section*{Acknowledgments}

The authors acknowledge that this work was supported by the Bavarian Ministry for Economic Affairs, Regional Development and Energy as part of a project to support the thematic development of the Institute for Cognitive Systems.

{
    \small
    \bibliographystyle{ieeenat_fullname}
    \bibliography{main}

@inproceedings{venkateswara2017deep,
  title={Deep hashing network for unsupervised domain adaptation},
  author={Venkateswara, Hemanth and Eusebio, Jose and Chakraborty, Shayok and Panchanathan, Sethuraman},
  booktitle={Proceedings of the IEEE Conference on Computer Vision and Pattern Recognition},
  pages={5018--5027},
  year={2017}
}

@article{dong2025trustvlm,

    title={To Trust Or Not To Trust Your Vision-Language Model's Prediction},

    author={Dong, Hao and Liu, Moru and Liang, Jian and Chatzi, Eleni and Fink, Olga},

    journal={arXiv preprint arXiv:2505.23745},

    year={2025}

}

@inproceedings{sun2025unseen,

  title={Unseen Visual Anomaly Generation},

  author={Sun, Han and Cao, Yunkang and Dong, Hao and Fink, Olga},

  booktitle={Proceedings of the Computer Vision and Pattern Recognition Conference},

  pages={25508--25517},

  year={2025}

}

@article{fink2026physics,
  title={From Physics to Machine Learning and Back: Part I-Learning with Inductive Biases in Prognostics and Health Management},
  author={Fink, Olga and Sharma, Vinay and Nejjar, Ismail and Von Krannichfeldt, Leandro and Garmaev, Sergei and Zhang, Zepeng and Wei, Amaury and Frusque, Gaetan and Forest, Florent and Zhao, Mengjie and others},
  journal={Reliability Engineering \& System Safety},
  pages={112213},
  year={2026},
  publisher={Elsevier}
}

@inproceedings{rauker2023toward,
  title={Toward transparent ai: A survey on interpreting the inner structures of deep neural networks},
  author={R{\"a}uker, Tilman and Ho, Anson and Casper, Stephen and Hadfield-Menell, Dylan},
  booktitle={2023 ieee conference on secure and trustworthy machine learning (satml)},
  pages={464--483},
  year={2023},
  organization={IEEE}
}

@article{jain2022distilling,
  title={Distilling model failures as directions in latent space},
  author={Jain, Saachi and Lawrence, Hannah and Moitra, Ankur and Madry, Aleksander},
  journal={arXiv preprint arXiv:2206.14754},
  year={2022}
}

@article{narayanaswamy2022know,
  title={Know your space: Inlier and outlier construction for calibrating medical OOD detectors},
  author={Narayanaswamy, Vivek and Mubarka, Yamen and Anirudh, Rushil and Rajan, Deepta and Spanias, Andreas and Thiagarajan, Jayaraman J},
  journal={arXiv preprint arXiv:2207.05286},
  year={2022}
}

@inproceedings{subramanyam2024decider,
  title={Decider: Leveraging foundation model priors for improved model failure detection and explanation},
  author={Subramanyam, Rakshith and Thopalli, Kowshik and Narayanaswamy, Vivek and Thiagarajan, Jayaraman J},
  booktitle={ECCV},
  year={2024},
}

@inproceedings{nguyen2025interpretable,
  title={Interpretable failure detection with human-level concepts},
  author={Nguyen, Kien X and Li, Tang and Peng, Xi},
  booktitle={Proceedings of the AAAI Conference on Artificial Intelligence},
  volume={39},
  number={25},
  pages={26326--26334},
  year={2025}
}

@article{scarlett2019introductory,
  title={An introductory guide to Fano's inequality with applications in statistical estimation},
  author={Scarlett, Jonathan and Cevher, Volkan},
  journal={arXiv preprint arXiv:1901.00555},
  year={2019}
}

@article{beaudry2011intuitive,
  title={An intuitive proof of the data processing inequality},
  author={Beaudry, Normand J and Renner, Renato},
  journal={arXiv preprint arXiv:1107.0740},
  year={2011}
}

@inproceedings{qiu2022detecting,
  title={Detecting misclassification errors in neural networks with a gaussian process model},
  author={Qiu, Xin and Miikkulainen, Risto},
  booktitle={AAAI},
  year={2022}
}

@article{hendrycks2016baseline,
  title={A baseline for detecting misclassified and out-of-distribution examples in neural networks},
  author={Hendrycks, Dan and Gimpel, Kevin},
  journal={arXiv preprint arXiv:1610.02136},
  year={2016}
}

@inproceedings{zhu2022rethinking,
  title={Rethinking confidence calibration for failure prediction},
  author={Zhu, Fei and Cheng, Zhen and Zhang, Xu-Yao and Liu, Cheng-Lin},
  booktitle={ECCV},
  year={2022},
}

@article{zeng2025towards,
  title={Towards Efficient and General-Purpose Few-Shot Misclassification Detection for Vision-Language Models},
  author={Zeng, Fanhu and Cheng, Zhen and Zhu, Fei and Zhang, Xu-Yao},
  journal={arXiv preprint arXiv:2503.20492},
  year={2025}
}

@article{jaeger2022call,
  title={A call to reflect on evaluation practices for failure detection in image classification},
  author={Jaeger, Paul F and L{\"u}th, Carsten T and Klein, Lukas and Bungert, Till J},
  journal={arXiv preprint arXiv:2211.15259},
  year={2022}
}

@ARTICLE{1054406,
  author={Chow, C.},
  journal={IEEE Transactions on Information Theory}, 
  title={On optimum recognition error and reject tradeoff}, 
  year={1970},
  volume={16},
  number={1},
  pages={41-46},}

@inproceedings{granese2021doctor,
  title={Doctor: A simple method for detecting misclassification errors},
  author={Granese, Federica and Romanelli, Marco and Gorla, Daniele and Palamidessi, Catuscia and Piantanida, Pablo},
  booktitle={NeurIPS},
  year={2021}
}

@inproceedings{pinto2022using,
  title={Using mixup as a regularizer can surprisingly improve accuracy \& out-of-distribution robustness},
  author={Pinto, Francesco and Yang, Harry and Lim, Ser Nam and Torr, Philip and Dokania, Puneet},
  booktitle={NeurIPS},
  year={2022}
}

@article{zhang2017mixup,
  title={mixup: Beyond empirical risk minimization},
  author={Zhang, Hongyi and Cisse, Moustapha and Dauphin, Yann N and Lopez-Paz, David},
  journal={arXiv preprint arXiv:1710.09412},
  year={2017}
}

@article{han2024unveiling,
  title={Unveiling AI's Blind Spots: An Oracle for In-Domain, Out-of-Domain, and Adversarial Errors},
  author={Han, Shuangpeng and Zhang, Mengmi},
  journal={arXiv preprint arXiv:2410.02384},
  year={2024}
}

@inproceedings{ConfidNet,
  title={Addressing failure prediction by learning model confidence},
  author={Corbi{\`e}re, Charles and Thome, Nicolas and Bar-Hen, Avner and Cord, Matthieu and P{\'e}rez, Patrick},
  booktitle={NeurIPS},
  year={2019}
}

@inproceedings{jiang2018trust,
  title={To trust or not to trust a classifier},
  author={Jiang, Heinrich and Kim, Been and Guan, Melody and Gupta, Maya},
  booktitle={NeurIPS},
  year={2018}
}

@inproceedings{moon2020confidence,
  title={Confidence-aware learning for deep neural networks},
  author={Moon, Jooyoung and Kim, Jihyo and Shin, Younghak and Hwang, Sangheum},
  booktitle={ICML},
  year={2020},
}

@article{soomro2012ucf101,
  title={UCF101: A dataset of 101 human actions classes from videos in the wild},
  author={Soomro, Khurram and Zamir, Amir Roshan and Shah, Mubarak},
  journal={arXiv preprint arXiv:1212.0402},
  year={2012}
}

@inproceedings{SSL,
    title = {Learning to Predict Trustworthiness with Steep Slope Loss},
    booktitle = {NeurIPS},
    author = {Luo, Yan and Wong, Yongkang and Kankanhalli, Mohan S and Zhao, Qi},
    year = {2021},
}

@inproceedings{selvaraju2017grad,
  title={Grad-cam: Visual explanations from deep networks via gradient-based localization},
  author={Selvaraju, Ramprasaath R and Cogswell, Michael and Das, Abhishek and Vedantam, Ramakrishna and Parikh, Devi and Batra, Dhruv},
  booktitle={Proceedings of the IEEE international conference on computer vision},
  pages={618--626},
  year={2017}
}

@InProceedings{pmlr-v80-kim18d,
  title = 	 {Interpretability Beyond Feature Attribution: Quantitative Testing with Concept Activation Vectors ({TCAV})},
  author =       {Kim, Been and Wattenberg, Martin and Gilmer, Justin and Cai, Carrie and Wexler, James and Viegas, Fernanda and sayres, Rory},
  booktitle = 	 {Proceedings of the 35th International Conference on Machine Learning},
  pages = 	 {2668--2677},
  year = 	 {2018},
  editor = 	 {Dy, Jennifer and Krause, Andreas},
  volume = 	 {80},
  series = 	 {Proceedings of Machine Learning Research},
  month = 	 {10--15 Jul},
  publisher =    {PMLR},
}

@article{adebayo2018sanity,
  title={Sanity checks for saliency maps},
  author={Adebayo, Julius and Gilmer, Justin and Muelly, Michael and Goodfellow, Ian and Hardt, Moritz and Kim, Been},
  journal={Advances in neural information processing systems},
  volume={31},
  year={2018}
}

@inproceedings{wu2022characterizing,
  title={Characterizing and overcoming the greedy nature of learning in multi-modal deep neural networks},
  author={Wu, Nan and Jastrzebski, Stanislaw and Cho, Kyunghyun and Geras, Krzysztof J},
  booktitle={ICML},
  year={2022},
}

@article{liu2024typicalness,
  title={Typicalness-Aware Learning for Failure Detection},
  author={Liu, Yijun and Cui, Jiequan and Tian, Zhuotao and Yang, Senqiao and He, Qingdong and Wang, Xiaoling and Su, Jingyong},
  journal={arXiv preprint arXiv:2411.01981},
  year={2024}
}

@inproceedings{li2024sure,
  title={SURE: SUrvey REcipes for building reliable and robust deep networks},
  author={Li, Yuting and Chen, Yingyi and Yu, Xuanlong and Chen, Dexiong and Shen, Xi},
  booktitle={CVPR},
  year={2024}
}

@article{liu2025fm,
  title={Extremely Simple Multimodal Outlier Synthesis for Out-of-Distribution Detection and Segmentation},
  author={Liu, Moru and Dong, Hao and Kelly, Jessica and Fink, Olga and Trapp, Mario},
  journal={arXiv preprint arXiv:2505.16985},
  year={2025}
}

@article{cheng2024breaking,
  title={Breaking the limits of reliable prediction via generated data},
  author={Cheng, Zhen and Zhu, Fei and Zhang, Xu-Yao and Liu, Cheng-Lin},
  journal={IJCV},
  year={2024},
}

@INPROCEEDINGS{OpenMix,
  author={Zhu, Fei and Cheng, Zhen and Zhang, Xu-Yao and Liu, Cheng-Lin},
  booktitle={CVPR}, 
  title={OpenMix: Exploring Outlier Samples for Misclassification Detection}, 
  year={2023},}

@article{dong2025mmdasurvey,

	author   = {Dong, Hao and Liu, Moru and Zhou, Kaiyang and Chatzi, Eleni and Kannala, Juho and Stachniss, Cyrill and Fink, Olga},

	title    = {Advances in Multimodal Adaptation and Generalization: From Traditional Approaches to Foundation Models},

	journal  = {arXiv preprint arXiv:2501.18592},

	year     = {2025},

}

@article{feng2022towards,
  title={Towards better selective classification},
  author={Feng, Leo and Ahmed, Mohamed Osama and Hajimirsadeghi, Hossein and Abdi, Amir},
  journal={arXiv preprint arXiv:2206.09034},
  year={2022}
}

@inproceedings{liu2020energy,
  title={Energy-based out-of-distribution detection},
  author={Liu, Weitang and Wang, Xiaoyun and Owens, John and Li, Yixuan},
  booktitle={NeurIPS},
  year={2020}
}

@inproceedings{he2016deep,
  title={Deep residual learning for image recognition},
  author={He, Kaiming and Zhang, Xiangyu and Ren, Shaoqing and Sun, Jian},
  booktitle={CVPR},
  pages={770--778},
  year={2016}
}

@INPROCEEDINGS{9053174,
  author={Chen, Honglie and Xie, Weidi and Vedaldi, Andrea and Zisserman, Andrew},
  booktitle={ICASSP}, 
  title={Vggsound: A Large-Scale Audio-Visual Dataset},
year = {2020}}

@InProceedings{Feichtenhofer_2019_ICCV,
author = {Feichtenhofer, Christoph and Fan, Haoqi and Malik, Jitendra and He, Kaiming},
title = {SlowFast Networks for Video Recognition},
booktitle = {ICCV},
year = {2019}
}

@article{liang2017enhancing,
  title={Enhancing the reliability of out-of-distribution image detection in neural networks},
  author={Liang, Shiyu and Li, Yixuan and Srikant, Rayadurgam},
  journal={arXiv preprint arXiv:1706.02690},
  year={2017}
}

@inproceedings{sun2022out,
  title={Out-of-distribution detection with deep nearest neighbors},
  author={Sun, Yiyou and Ming, Yifei and Zhu, Xiaojin and Li, Yixuan},
  booktitle={ICML},
  year={2022},
}

@inproceedings{carreira2017quo,
  title={Quo vadis, action recognition? a new model and the kinetics dataset},
  author={Carreira, Joao and Zisserman, Andrew},
  booktitle={CVPR},
  year={2017}
}

@inproceedings{wang2016temporal,
  title={Temporal segment networks: Towards good practices for deep action recognition},
  author={Wang, Limin and Xiong, Yuanjun and Wang, Zhe and Qiao, Yu and Lin, Dahua and Tang, Xiaoou and Van Gool, Luc},
  booktitle={ECCV},
  year={2016},
}

@inproceedings{wang2018tienet,
  title={Tienet: Text-image embedding network for common thorax disease classification and reporting in chest x-rays},
  author={Wang, Xiaosong and Peng, Yifan and Lu, Le and Lu, Zhiyong and Summers, Ronald M},
  booktitle={CVPR},
  year={2018}
}

@article{feng2020deep,
  title={Deep multi-modal object detection and semantic segmentation for autonomous driving: Datasets, methods, and challenges},
  author={Feng, Di and Haase-Sch{\"u}tz, Christian and Rosenbaum, Lars and Hertlein, Heinz and Glaeser, Claudius and Timm, Fabian and Wiesbeck, Werner and Dietmayer, Klaus},
  journal={IEEE Transactions on Intelligent Transportation Systems},
  volume={22},
  number={3},
  pages={1341--1360},
  year={2020},
  publisher={IEEE}
}

@inproceedings{lee2018simple,
  title={A simple unified framework for detecting out-of-distribution samples and adversarial attacks},
  author={Lee, Kimin and Lee, Kibok and Lee, Honglak and Shin, Jinwoo},
  booktitle={NeurIPS},
  year={2018}
}

@inproceedings{zhang2023mixture,
  title={Mixture outlier exposure: Towards out-of-distribution detection in fine-grained environments},
  author={Zhang, Jingyang and Inkawhich, Nathan and Linderman, Randolph and Chen, Yiran and Li, Hai},
  booktitle={Proceedings of the IEEE/CVF Winter Conference on Applications of Computer Vision},
  pages={5531--5540},
  year={2023}
}

@article{hendrycks2018deep,
  title={Deep anomaly detection with outlier exposure},
  author={Hendrycks, Dan and Mazeika, Mantas and Dietterich, Thomas},
  journal={arXiv preprint arXiv:1812.04606},
  year={2018}
}

@inproceedings{yu2019unsupervised,
  title={Unsupervised out-of-distribution detection by maximum classifier discrepancy},
  author={Yu, Qing and Aizawa, Kiyoharu},
  booktitle={ICCV},
  year={2019}
}

@inproceedings{yang2021semantically,
  title={Semantically coherent out-of-distribution detection},
  author={Yang, Jingkang and Wang, Haoqi and Feng, Litong and Yan, Xiaopeng and Zheng, Huabin and Zhang, Wayne and Liu, Ziwei},
  booktitle={ICCV},
  year={2021}
}

@article{rasenberg2020alignment,
  title={Alignment in multimodal interaction: An integrative framework},
  author={Rasenberg, Marlou and {\"O}zy{\"u}rek, Asli and Dingemanse, Mark},
  journal={Cognitive science},
  volume={44},
  number={11},
  year={2020},
}

@inproceedings{simonyan2014two,
  title={Two-stream convolutional networks for action recognition in videos},
  author={Simonyan, Karen and Zisserman, Andrew},
  booktitle={NeurIPS},
  year={2014}
}

@inproceedings{behley2019semantickitti,
  title={Semantickitti: A dataset for semantic scene understanding of lidar sequences},
  author={Behley, Jens and Garbade, Martin and Milioto, Andres and Quenzel, Jan and Behnke, Sven and Stachniss, Cyrill and Gall, Jurgen},
  booktitle={ICCV},
  year={2019}
}

@inproceedings{pmf,
  title={Perception-aware multi-sensor fusion for 3d lidar semantic segmentation},
  author={Zhuang, Zhuangwei and Li, Rong and Jia, Kui and Wang, Qicheng and Li, Yuanqing and Tan, Mingkui},
  booktitle={ICCV},
  year={2021}
}

@INPROCEEDINGS{10203747,
  author={Liu, Xixi and Lochman, Yaroslava and Zach, Christopher},
  booktitle={CVPR}, 
  title={GEN: Pushing the Limits of Softmax-Based Out-of-Distribution Detection}, 
  year={2023},}

@article{yang2021oodsurvey,
    title={Generalized Out-of-Distribution Detection: A Survey},
    author={Yang, Jingkang and Zhou, Kaiyang and Li, Yixuan and Liu, Ziwei},
    journal={arXiv preprint arXiv:2110.11334},
    year={2021}
}

@article{zhang2023openood,
  title={OpenOOD v1.5: Enhanced Benchmark for Out-of-Distribution Detection},
  author={Zhang, Jingyang and Yang, Jingkang and Wang, Pengyun and Wang, Haoqi and Lin, Yueqian and Zhang, Haoran and Sun, Yiyou and Du, Xuefeng and Li, Yixuan and Liu, Ziwei and Chen, Yiran and Li, Hai},
  journal={arXiv preprint arXiv:2306.09301},
  year={2023}
}

@misc{2020mmaction2,
    title={OpenMMLab's Next Generation Video Understanding Toolbox and Benchmark},
    author={MMAction2 Contributors},
    howpublished = {\url{https://github.com/open-mmlab/mmaction2}},
    year={2020}
}

@inproceedings{chan2021entropy,
  title={Entropy maximization and meta classification for out-of-distribution detection in semantic segmentation},
  author={Chan, Robin and Rottmann, Matthias and Gottschalk, Hanno},
  booktitle={ICCV},
  year={2021}
}

@inproceedings{Adam,
    author = {Diederik P Kingma and Jimmy Ba},
    title   = {Adam: A method for stochastic optimization},
    booktitle = {ICLR},
    year    = {2015}
}

@inproceedings{dong2024multiood,
  title={Multiood: Scaling out-of-distribution detection for multiple modalities},
  author={Dong, Hao and Zhao, Yue and Chatzi, Eleni and Fink, Olga},
  booktitle={NeurIPS},
  year={2024}
}

@article{hendrycks2019scaling,
  title={Scaling out-of-distribution detection for real-world settings},
  author={Hendrycks, Dan and Basart, Steven and Mazeika, Mantas and Zou, Andy and Kwon, Joe and Mostajabi, Mohammadreza and Steinhardt, Jacob and Song, Dawn},
  journal={arXiv preprint arXiv:1911.11132},
  year={2019}
}

@article{li2024dpu,
  title={Dpu: Dynamic prototype updating for multimodal out-of-distribution detection},
  author={Li, Shawn and Gong, Huixian and Dong, Hao and Yang, Tiankai and Tu, Zhengzhong and Zhao, Yue},
  journal={arXiv preprint arXiv:2411.08227},
  year={2024}
}

@INPROCEEDINGS{9578249,
  author={Di Biase, Giancarlo and Blum, Hermann and Siegwart, Roland and Cadena, Cesar},
  booktitle={CVPR}, 
  title={Pixel-wise Anomaly Detection in Complex Driving Scenes}, 
  year={2021},}

@inproceedings{zhou2022rethinking,
  title={Rethinking reconstruction autoencoder-based out-of-distribution detection},
  author={Zhou, Yibo},
  booktitle={CVPR},
  year={2022}
}

@inproceedings{tian2022pixel,
  title={Pixel-wise energy-biased abstention learning for anomaly segmentation on complex urban driving scenes},
  author={Tian, Yu and Liu, Yuyuan and Pang, Guansong and Liu, Fengbei and Chen, Yuanhong and Carneiro, Gustavo},
  booktitle={ECCV},
  year={2022},
}

@InProceedings{pmlr-v70-guo17a,
  title = 	 {On Calibration of Modern Neural Networks},
  author =       {Chuan Guo and Geoff Pleiss and Yu Sun and Kilian Q. Weinberger},
  booktitle = 	 {ICML},
  year = 	 {2017},
}

@inproceedings{lakshminarayanan2017simple,
  title={Simple and scalable predictive uncertainty estimation using deep ensembles},
  author={Lakshminarayanan, Balaji and Pritzel, Alexander and Blundell, Charles},
  booktitle={NeurIPS},
  year={2017}
}

@inproceedings{kuehne2011hmdb,
  title={HMDB: a large video database for human motion recognition},
  author={Kuehne, Hildegard and Jhuang, Hueihan and Garrote, Est{\'\i}baliz and Poggio, Tomaso and Serre, Thomas},
  booktitle={ICCV},
  year={2011},
}

@article{kay2017kinetics,
  title={The kinetics human action video dataset},
  author={Kay, Will and Carreira, Joao and Simonyan, Karen and Zhang, Brian and Hillier, Chloe and Vijayanarasimhan, Sudheendra and Viola, Fabio and Green, Tim and Back, Trevor and Natsev, Paul and others},
  journal={arXiv preprint arXiv:1705.06950},
  year={2017}
}

@inproceedings{dong2023SimMMDG,

    title={Sim{MMDG}: A Simple and Effective Framework for Multi-modal Domain Generalization},

    author={Dong, Hao and Nejjar, Ismail and Sun, Han and Chatzi, Eleni and Fink, Olga},

    booktitle={NeurIPS},

    year={2023}

}

@INPROCEEDINGS{Damen2018EPICKITCHENS,
   title={Scaling Egocentric Vision: The EPIC-KITCHENS Dataset},
   author={Damen, Dima and Doughty, Hazel and Farinella, Giovanni Maria  and Fidler, Sanja and 
           Furnari, Antonino and Kazakos, Evangelos and Moltisanti, Davide and Munro, Jonathan 
           and Perrett, Toby and Price, Will and Wray, Michael},
   booktitle={ECCV},
  year={2018}
}
}

\clearpage
\newpage
\appendix

\newpage


\section{Theoretical Analysis on Confidence Degradation}

\begin{itemize}
    \item Let $Y$ be the discrete label variable taking values in $\{1, \dots, K\}$.
    \item Let $X_S$ denote the collection of modality inputs whose indices lie in $S \subseteq \{1, \dots, M\}$. In particular $X_M$ is the full multimodal input and $X_T$ with $T \subset S$ denotes a subset (less information).
    \item The conditional entropy of $Y$ on $X$ is $H(Y \mid X)$.
\end{itemize}

\begin{restatable}{theorem}{low}
\label{thm:low}
Adding more modalities cannot increase the conditional entropy:
$$H(Y \mid X_S) \le H(Y \mid X_T), \quad \text{for } T \subset S.$$
\end{restatable}

\noindent\textbf{Proof.} If $T \subset S$, then by the data processing inequality~\cite{beaudry2011intuitive}, we have $I(Y; X_T) \le I(Y; X_S)$, meaning that adding modalities cannot decrease the mutual information between the input and the label.
Since $I(Y; X) = H(Y) - H(Y \mid X)$ and $H(Y)$ is fixed, it follows directly that $ H(Y \mid X_S) \le H(Y \mid X_T). \quad$ $\blacksquare$

From Theorem 1, we know that when all modalities are predictive of the target, adding additional input information (i.e., more modalities) cannot increase the conditional entropy. Consequently, the model’s uncertainty should not increase as more modalities are provided. Confidence degradation, where predictions become less confident after adding modalities, therefore reflects an undesirable and theoretically inconsistent behavior in multimodal learning.

\begin{restatable}{theorem}{low2}
\label{thm:low2}
Confidence degradation is linked to a higher theoretical bound on prediction error.
\end{restatable}

\noindent\textbf{Proof.} Let $P_e = Pr(Y \neq \hat{Y})$ denote the probability of prediction error. By Fano's inequality~\cite{scarlett2019introductory}, the error rate is bounded by the true conditional entropy:
$$ P_e \ge \frac{H(Y|X) - 1}{\log_2(|Y|)}. $$
Confidence degradation refers to the phenomenon where the fused multimodal prediction has \emph{lower} confidence than at least one of its unimodal predictions. Lower predictive confidence implies greater uncertainty about the output distribution, which corresponds to an increase in $H(Y|X)$.  According to Fano's inequality, a higher conditional entropy raises the minimum theoretical bound on error. Thus, while not a direct cause, confidence degradation reflects a state where the model's perceived uncertainty is higher, which is strongly associated with a higher potential for misclassification. $\quad \blacksquare$

Theorem 2 formalizes the key phenomenon we observe in practice: failures in multimodal systems frequently occur together with confidence degradation.

\section{Broader Impact, Limitations, and Future Work}

\noindent\textbf{Broader Impact.} The ACR framework presented in this work has the potential to generate a significant positive societal impact. As AI systems, especially multimodal ones, are increasingly integrated into safety-critical applications such as autonomous driving, medical diagnosis, and industrial robotics, the ability to reliably detect potential failures (i.e., misclassifications) is paramount. ACR directly contributes to enhancing the safety and trustworthiness of these systems by providing a dedicated mechanism to identify and flag uncertain predictions before they can lead to adverse outcomes. This can foster greater public trust and accelerate the responsible adoption of AI technologies in areas where reliability is non-negotiable. Furthermore, by improving the understanding of failure modes in multimodal models, such as the identified confidence degradation phenomenon, our work can guide the development of more robust and dependable AI systems across various domains, ultimately leading to safer and more effective human-AI collaboration. 

\noindent\textbf{Limitations.}  While ACR demonstrates strong gains across multiple datasets and modalities, it has several limitations. First, while we demonstrate robustness under certain distribution shifts and OOD scenarios, the behavior of ACR against sophisticated adversarial attacks specifically designed to fool the misclassification detector itself remains an open area. 
Second, our current study primarily focuses on specific types of modalities (e.g., video, optical flow, and audio). The generalization of ACR to a very broad range of disparate modalities (e.g., text with medical imaging, or sensor data with acoustic signals) without specific adaptations warrants further investigation.

\noindent\textbf{Future Work.} Building upon the contributions of this paper, several avenues for future research present themselves.
First, we plan to explore the integration of ACR with online learning and continual learning paradigms. This would allow the misclassification detector to adapt dynamically to evolving data distributions and novel failure modes encountered during real-world deployment, which is crucial for long-term operational reliability.
Second, extending and evaluating ACR across a wider spectrum of modalities (e.g., language, audio, tabular data) and a more diverse set of complex, safety-critical tasks (e.g., real-time robotic interaction) is a key direction.
Finally, investigating the robustness of ACR against targeted adversarial attacks and developing defense mechanisms will be crucial for deployment in adversarial settings. 

\section{Related Work}
\subsection{Failure Detection}
{Foundational work on failure prediction originates from selective
classification, which formalizes the trade-off between prediction accuracy and
abstention under uncertainty~\citep{1054406}. In the deep learning era, this
concept has re-emerged as \emph{failure detection} (FD), where the goal is to
identify test samples on which a model is likely to be incorrect~\citep{qiu2022detecting,dong2025trustvlm}.
Thresholding the maximum softmax probability (MSP)~\citep{hendrycks2016baseline}
remains a strong baseline, yet its vulnerability to overconfidence limits its
reliability~\citep{pinto2022using,qiu2022detecting}.
Beyond MSP, a large body of work aims to more effectively estimate prediction
risk.}


One stream of research focuses on training auxiliary modules, such as heads or separate networks, to explicitly predict correctness based on intermediate features or model logits~\citep{ConfidNet,jiang2018trust,granese2021doctor}. For instance, ConfidNet~\citep{ConfidNet} introduces a dedicated confidence estimation head that operates on penultimate-layer features. 
A distinct line of research adopts an integrated approach, jointly optimizing FD functionalities with the primary model's training objective. Representative techniques in this category include improving the separability of correct and incorrect representations~\citep{SSL}, enhancing confidence ranking~\citep{moon2020confidence}, regularizing based on sample typicality~\citep{liu2024typicalness}, and enforcing flatter loss landscapes around misclassified examples~\citep{zhu2022rethinking}. Concurrently, data augmentation strategies have proven effective, primarily by exposing the model to synthesized failure cases during training~\citep{OpenMix, li2024sure, cheng2024breaking, han2024unveiling}. However, all previous approaches were designed for unimodal scenarios, without accounting for the interaction and complementary nature of diverse modalities.

\subsection{Out-of-Distribution Detection}

OOD detection shares a similar objective with FD but addresses fundamentally different challenges~\cite{narayanaswamy2022know}. Specifically, OOD detection aims to identify test samples that exhibit semantic shifts from the training distribution, typically without compromising in-distribution (ID) classification accuracy. 
OOD detection has been extensively investigated in recent years~\citep{yang2021oodsurvey, zhang2023openood}, encompassing diverse approaches such as post-hoc scoring~\citep{hendrycks2016baseline, liang2017enhancing, liu2020energy}, feature-based techniques~\citep{lee2018simple, sun2022out}, outlier exposure~\citep{hendrycks2018deep, yu2019unsupervised, yang2021semantically}, and reconstruction-based methods~\citep{9578249, zhou2022rethinking}. Multimodal OOD detection~\citep{dong2024multiood, li2024dpu} is also an emerging research area. 
Although OOD detection methods are often employed as baselines for FD, recent studies~\citep{zhu2022rethinking,jaeger2022call,OpenMix} demonstrate that techniques optimized for OOD detection generally exhibit suboptimal performance on FD tasks. This finding underscores the necessity for developing specialized FD approaches.

\subsection{{Interpretability-Based Failure Prediction}}

A complementary line of work investigates predicting failures by analyzing a model’s internal reasoning rather than relying solely on output confidence~\cite{subramanyam2024decider,jain2022distilling}. Recent methods~\citep{rauker2023toward,nguyen2025interpretable} leverage post-hoc interpretability techniques (e.g., Grad-CAM~\cite{selvaraju2017grad}, saliency~\cite{adebayo2018sanity}, concept activation~\cite{pmlr-v80-kim18d}) and train auxiliary meta-predictors to detect unreliable explanations. These approaches identify patterns such as noisy attention, background-biased focus, and lack of class-discriminative localization, which often correlate with misclassifications. While informative, interpretability-driven failure prediction requires generating explanations for every test sample, incurring significant computational overhead, and remains decoupled from the model’s training dynamics. In contrast, ACR embeds failure awareness directly into training and enables failure prediction in a single forward pass, eliminating the cost and complexity associated with post-hoc explanations.





\section{Further Details on Datasets}

Our framework is primarily evaluated on four action recognition datasets from the MultiOOD benchmark~\citep{dong2024multiood}: HMDB51~\citep{kuehne2011hmdb}, EPIC-Kitchens~\citep{Damen2018EPICKITCHENS}, HAC~\citep{dong2023SimMMDG}, and Kinetics-600~\citep{kay2017kinetics}. For evaluating multimodal Out-of-Distribution (OOD) detection in ablation studies, we additionally employ the UCF101 dataset~\citep{simonyan2014two}, also from the MultiOOD benchmark. \cref{fig:dataset} illustrates the four main datasets used for multimodal failure detection.

\begin{figure}[t]
  \centering
  \includegraphics[width=\linewidth]{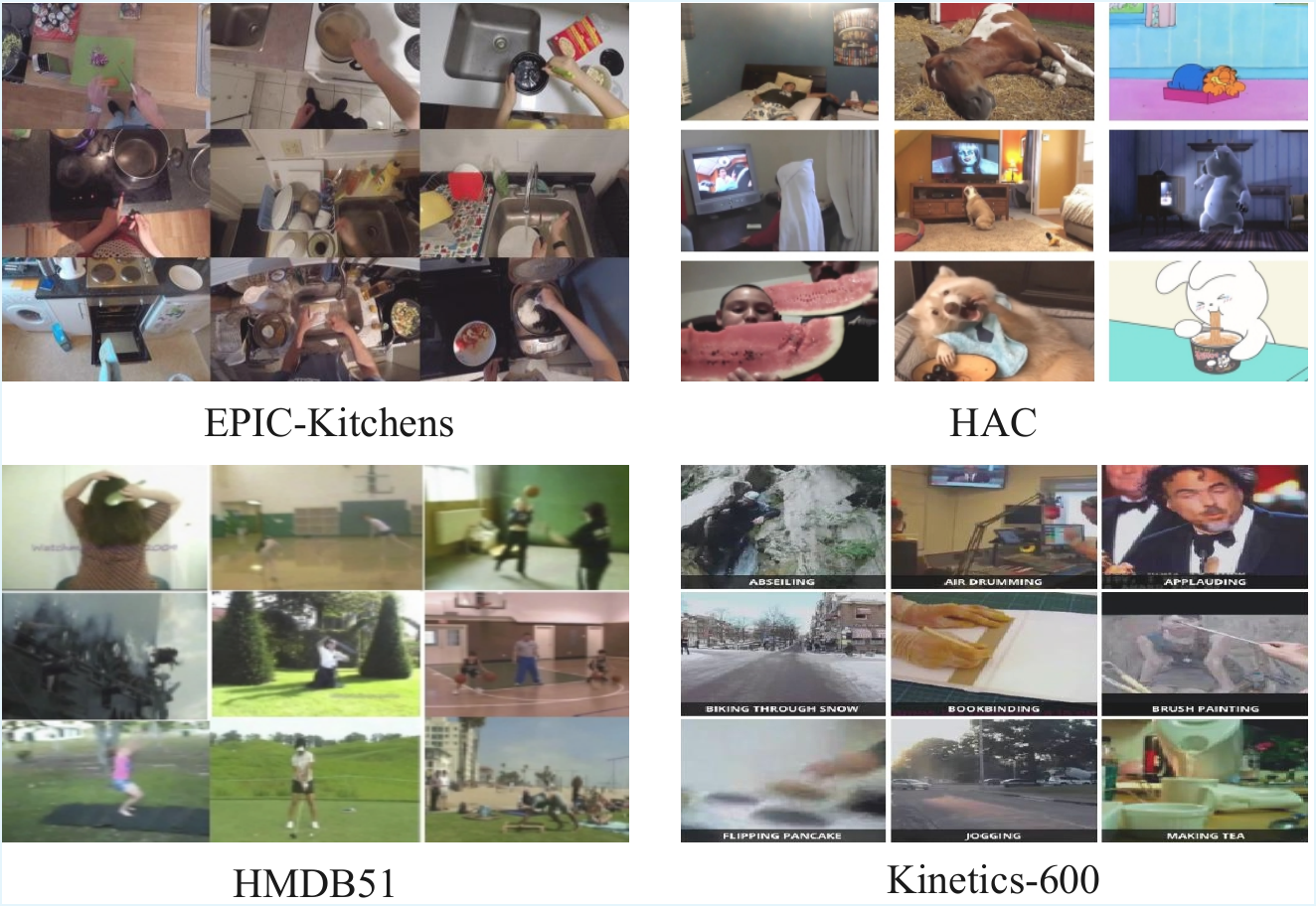}
  \vspace{-0.7cm}
  \caption{An illustration of the four main datasets we used for multimodal failure detection.}
  \label{fig:dataset}
\end{figure}

\textbf{HMDB51} is an action recognition dataset comprising $6{,}766$ video clips distributed across $51$ action categories. The videos are sourced from digitized movies and YouTube. This dataset includes both video and optical flow modalities.
\textbf{EPIC-Kitchens} is an egocentric video dataset capturing daily kitchen activities recorded by $32$ participants. For our experiments, we utilize a subset of $4{,}871$ video clips from participant P22, encompassing eight common actions (\emph{put}, \emph{take}, \emph{open}, \emph{close}, \emph{wash}, \emph{cut}, \emph{mix}, and \emph{pour}). The dataset provides video and optical flow modalities.  
\textbf{Kinetics-600} is a large-scale action recognition dataset containing approximately $480{,}000$ $10$-second clips distributed across $600$ action classes. Following~\citep{dong2024multiood}, we utilize a subset of $100$ classes, resulting in $24{,}981$ video clips for our study. This dataset offers video, audio, and optical flow modalities.
\textbf{HAC} contains $3{,}381$ video clips featuring seven action categories (e.g., \emph{sleeping}, \emph{watching TV}, \emph{eating}, \emph{running}) performed by humans, animals, and cartoon characters. The dataset includes video, optical flow, and audio modalities.
\textbf{UCF101} is a diverse video action recognition dataset consisting of $13{,}320$ clips across $101$ action classes. The videos, sourced from YouTube, exhibit significant variation in camera motion, object appearance, scale, pose, viewpoint, and background. This dataset provides video and optical flow modalities.

\section{More Implementation Details}
\noindent\textbf{Pseudo Code for Multimodal Feature Swapping.} We provide the pseudo code for multimodal outlier synthesis in \cref{alg:failure_synthesis}, where we dynamically swap multimodal feature embeddings and assign them corresponding soft labels. MFS naturally generalizes beyond the two-modality setting. In \cref{alg:failure_synthesis2}, we illustrate how to extend MFS to three modalities. We extend MFS to this setting by randomly selecting a start index for each modality and swapping them cyclically: $\mathbf{E}^1 \rightarrow \mathbf{E}^2$, $\mathbf{E}^2 \rightarrow \mathbf{E}^3$, and $\mathbf{E}^3 \rightarrow \mathbf{E}^1$, thus creating outlier features for all three modalities.  This generalized formulation enables flexible synthesis of multimodal failure cases, capturing higher-order inconsistencies (e.g., tri-modal conflicts) and scaling effectively to heterogeneous sensor configurations commonly found in real-world systems.

\begin{algorithm}[t]
\caption{Multimodal Feature Swapping}
\label{alg:failure_synthesis}
\begin{algorithmic}
\STATE {\bfseries Input:} ID feature $\mathbf{E} = [\mathbf{E}^1, \mathbf{E}^2]$, where $\mathbf{E}^1$ and $\mathbf{E}^2$ are from modality $1$ and $2$; minimum and maximum number $n_{\text{min}}$ and $n_{\text{max}}$ for swapping; ground-truth one-hot label $\mathbf{y}_{\text{true}}$ for $\mathbf{E}$ and  $\mathbf{y}_{\text{outlier}} = C+1$.

\STATE {\bfseries Pseudo Code:}


\STATE \quad Sample $n_{\text{swap}} \sim \mathcal{U}(n_{\text{min}}, n_{\text{max}})$, $\lambda = \frac{n_{\text{swap}}}{n_{\text{max}}}$;
\STATE \quad Randomly select start indices $s_1$, $s_2$ for modality $1$ and $2$;
\STATE \quad Clone features $\widetilde{\mathbf{E}}^{1} \leftarrow \mathbf{E}^1$, $\widetilde{\mathbf{E}}^{2} \leftarrow \mathbf{E}^2$;
\STATE \quad Swap $n_{\text{swap}}$ dimensions across modalities:
\[
\widetilde{\mathbf{E}}^{1}[ s_1{:}s_1{+}n_{\text{swap}}] \leftarrow \mathbf{E}^2[ s_2{:}s_2{+}n_{\text{swap}}]
\]
\[
\widetilde{\mathbf{E}}^{2}[ s_2{:}s_2{+}n_{\text{swap}}] \leftarrow \mathbf{E}^1[ s_1{:}s_1{+}n_{\text{swap}}]
\]


\STATE \quad Generate label for outlier feature $\mathbf{y}_{\text{swapped}} = (1 - \lambda) \mathbf{y}_{\text{true}} + \lambda \mathbf{y}_{\text{outlier}}$.

\STATE {\bfseries Output:} Multimodal outlier feature $\mathbf{E}_o = [\widetilde{\mathbf{E}}^{1}, \widetilde{\mathbf{E}}^{2}]$ and label $\mathbf{y}_{\text{mixed}}$.
\end{algorithmic}
\end{algorithm}
\begin{algorithm}[t]
\caption{Multimodal Feature Swapping with Three Modalities}
\label{alg:failure_synthesis2}
\begin{algorithmic}
\STATE {\bfseries Input:} ID feature $\mathbf{E} = [\mathbf{E}^1, \mathbf{E}^2, \mathbf{E}^3]$, where $\mathbf{E}^1$, $\mathbf{E}^2$, and $\mathbf{E}^3$ are from modality $1$, $2$, and $3$; minimum and maximum number $n_{\text{min}}$ and $n_{\text{max}}$ for swapping; ground-truth one-hot label $\mathbf{y}_{\text{true}}$ for $\mathbf{E}$ and  $\mathbf{y}_{\text{outlier}} = C+1$.

\STATE {\bfseries Pseudo Code:}


\STATE \quad Sample $n_{\text{swap}} \sim \mathcal{U}(n_{\text{min}}, n_{\text{max}})$, $\lambda = \frac{n_{\text{swap}}}{n_{\text{max}}}$;
\STATE \quad Randomly select start indices $s_1$, $s_2$, $s_3$ for modality $1$, $2$, and $3$;
\STATE \quad Clone features $\widetilde{\mathbf{E}}^{1} \leftarrow \mathbf{E}^1$, $\widetilde{\mathbf{E}}^{2} \leftarrow \mathbf{E}^2$, $\widetilde{\mathbf{E}}^{3} \leftarrow \mathbf{E}^3$;
\STATE \quad Swap $n_{\text{swap}}$ dimensions across modalities:
\[
\widetilde{\mathbf{E}}^{1}[ s_1{:}s_1{+}n_{\text{swap}}] \leftarrow \mathbf{E}^3[ s_3{:}s_3{+}n_{\text{swap}}]
\]
\[
\widetilde{\mathbf{E}}^{2}[ s_2{:}s_2{+}n_{\text{swap}}] \leftarrow \mathbf{E}^1[ s_1{:}s_1{+}n_{\text{swap}}]
\]
\[
\widetilde{\mathbf{E}}^{3}[ s_3{:}s_3{+}n_{\text{swap}}] \leftarrow \mathbf{E}^2[ s_2{:}s_2{+}n_{\text{swap}}]
\]


\STATE \quad Generate label for outlier feature $\mathbf{y}_{\text{swapped}} = (1 - \lambda) \mathbf{y}_{\text{true}} + \lambda \mathbf{y}_{\text{outlier}}$.

\STATE {\bfseries Output:} Multimodal outlier feature $\mathbf{E}_o = [\widetilde{\mathbf{E}}^{1}, \widetilde{\mathbf{E}}^{2}, \widetilde{\mathbf{E}}^{3}]$ and label $\mathbf{y}_{\text{mixed}}$.
\end{algorithmic}
\end{algorithm}

\noindent\textbf{Extension of ACL to More Modalities.}
Our framework is not limited to two modalities and can be easily extended to $M$ modalities. Given a training sample $\bx$ with $M$ modalities, we obtain prediction confidence score $\textit{conf}$ from the combined embeddings of all modalities, and $\textit{conf}_1$, $\textit{conf}_2$, ..., $\textit{conf}_M$ from each modality. The Adaptive Confidence Loss can then be defined as:
\begin{equation}
  \mathcal{L}_{\text{acl}}
  =\frac{1}{M} \sum_{i=1}^{M} \max(0, \textit{conf}_i - \textit{conf}) .
\end{equation}

\section{More Ablation Studies}

\noindent\textbf{Parameter Sensitivity.} 
We evaluate the sensitivity of our framework to two key hyperparameters using the HMDB51 dataset. First, the maximum swapping dimension \( n_{max} \) for Multimodal Feature Swapping (MFS) was varied among $128$, $256$, and $512$, with results presented in Table \ref{tab:FD-cutmix}. An \( n_{max} \) value of $256$ yielded the optimal balance, achieving robust performance across all evaluation metrics.
Subsequently, with \( n_{max} \) fixed at $256$, the weight $\lambda_{\text{acl}}$ for Adaptive Confidence Loss (ACL) was evaluated over the set $0.2$, $0.5$, $1.0$, and $2.0$ (detailed in Table \ref{tab:FD-ada_conf_diff}). A value of \( \lambda_{\text{acl}} = 2.0 \) consistently delivered the strongest FD performance.
Importantly, the framework's performance remained stable across both parameter sweeps, underscoring its robustness to variations in these hyperparameters. 


\setlength{\tabcolsep}{2pt}
\begin{table}[t]
\begin{minipage}[t]{\columnwidth}
\centering
\begin{adjustbox}{width=0.65\linewidth}
\begin{tabular}{l | c c c c}
\hline
\multicolumn{1}{c|}{} & AURC$\downarrow$ & AUROC$\uparrow$ & FPR95$\downarrow$ & ACC$\uparrow$ \\
\hline

128         &29.08& 88.27&49.57  & \multicolumn{1}{c}{86.66} \\

256        & 25.11& 90.55 &46.22  & \multicolumn{1}{c}{86.43}  \\

512 & 25.34& 90.98 &43.90  & \multicolumn{1}{c}{85.97}\\

\hline
\end{tabular}
\end{adjustbox}

\vspace{-0.2cm}
\caption{Effect of $n_{max}$ in MFS.}

\label{tab:FD-cutmix}
\end{minipage}
\begin{minipage}[t]{\columnwidth}

\centering
\begin{adjustbox}{width=0.65\linewidth}
\begin{tabular}{l | c c c c}
\hline
\multicolumn{1}{c|}{} & AURC$\downarrow$ & AUROC$\uparrow$ & FPR95$\downarrow$ & ACC$\uparrow$ \\
\hline

0.2         &24.42& 90.19&46.15  & \multicolumn{1}{c}{86.66} \\

0.5        & 24.18&90.28 &46.22  & \multicolumn{1}{c}{87.00}  \\

1.0 & 23.11& 90.93&42.11  & \multicolumn{1}{c}{86.55}\\

2.0 & 19.97& 92.02 &41.96  & \multicolumn{1}{c}{87.23}\\

\hline
\end{tabular}
\end{adjustbox}

\vspace{-0.2cm}
\caption{Effect of weight $\lambda_{\text{acl}}$ for ACL.}

\label{tab:FD-ada_conf_diff}
\end{minipage}
\end{table}












\noindent\textbf{Generalize to Image Classification.} We evaluated on Office-Home Dataset~\cite{venkateswara2017deep} by treating images from distinct domains (i.e., Art and RealWorld) as different modalities and fusing them. We show that ACR can be generalized to the image classification task in \cref{tab:office}.

\begin{table}[t!]
\setlength\tabcolsep{3.0pt}
\centering
\resizebox{0.7\linewidth}{!}{%
\begin{tabular}{|l|c|c|c|c|}
\hline
 & AURC$\downarrow$ & AUROC$\uparrow$ & FPR95$\downarrow$ & ACC$\uparrow$ \\ \hline
MSP &  30.24 &86.63   & 62.75   & 86.03  \\ \hline
 ACR  &  \textbf{15.38} & \textbf{87.79}  & \textbf{58.06}   &  \textbf{91.51} \\ \hline
\end{tabular}
}
\vspace{-0.2cm}
\caption{Results on Office-Home dataset by fusing images from distinct domains.}
\label{tab:office}
\end{table}

\section{Further Discussions on Proposed Modules}
\label{sec:dis}







\subsection{Mechanism of ACL}
In Sec. 2.2, we show failures in multimodal systems frequently coincide with confidence degradation. We provided theoretical analysis in Appendix Sec. 1 and show that when distinct modalities are predictive of a target, adding modalities cannot increase conditional entropy. Hence, we introduce ACL to enforce better information integration, ensuring that fused evidence yields more separable confidence between correct and incorrect samples. The mechanism relies on the interplay between ACL and the cross-entropy loss ($\mathcal{L}_{\text{cls}}$). When predictions are \textit{\textbf{correct}}, $\mathcal{L}_{\text{cls}}$ and ACL work together to drive the fused confidence higher.
On \textit{\textbf{incorrect}} samples, $\mathcal{L}_{cls}$ pushes fused confidence down. If a unimodal branch remains overconfident, ACL remains high. To minimize total loss, gradients propagate to suppress the overconfident unimodal encoder, as the fused head cannot rise without violating $\mathcal{L}_{cls}$. This ensures that \textit{\textbf{ACL boosts confidence only for correct predictions while avoiding overconfidence in incorrect predictions (\cref{tab:merged_confidence})}}, hence (together with MFS) widening the confidence separation that directly benefits failure detection. 
\begin{table}[t]
\centering
\setlength{\tabcolsep}{3pt}
\resizebox{\columnwidth}{!}{
\begin{tabular}{lcccccccc}
\toprule
\multirow{2}{*}{\textbf{Method}} 
& \multicolumn{4}{c}{\textbf{Correct Predictions} ($\uparrow$)} 
& \multicolumn{4}{c}{\textbf{Incorrect Predictions} ($\downarrow$)} \\
\cmidrule(lr){2-5} 
\cmidrule(lr){6-9}
& HMDB51 & EPIC & HAC & Kinetics 
& HMDB51 & EPIC & HAC & Kinetics \\
\midrule
w/o ACL 
& 0.95 & 0.91 & 0.94 & 0.92 
& 0.80 & 0.73 & 0.74 & 0.64 \\

w/ ACL 
& \textbf{0.96} & \textbf{0.93} & \textbf{0.96} & \textbf{0.93}
& \textbf{0.63} & \textbf{0.69} & \textbf{0.62} & \textbf{0.57} \\
\bottomrule
\end{tabular}
}
\vspace{-0.2cm}
\caption{Average fused confidence on \textbf{correct} and \textbf{incorrect} predictions.}
\label{tab:merged_confidence}
\end{table}

\noindent\textbf{How Adaptive Confidence Loss Addresses Unimodal Overconfidence.} Imagine a situation where one modality (e.g., audio) is ambiguous or corrupted, causing its corresponding unimodal network to make a confidently wrong prediction (high $\textit{conf}_1$). The other modality (e.g., video) provides clear evidence for the correct class, leading to a correct, high-confidence prediction ($\textit{conf}_2$). A standard fusion model might struggle to integrate these conflicting signals. If the fusion process results in a moderate fused confidence $\textit{conf}$ that is lower than the erroneously high $\textit{conf}_1$, it will incur a large ACL penalty from the $\max(0, \textit{conf}_1 - \textit{conf})$ term. To minimize this loss, the model can't easily force the fused confidence up without a good reason, as that would be penalized by the cross-entropy loss term $\mathcal{L}_{\text{cls}}$ if the prediction is wrong. Instead, a more effective way to reduce the ACL loss is to lower the confidence of the overconfident unimodal prediction ($\textit{conf}_1$). Through repeated exposure to such conflicting examples during training, the unimodal feature extractor learns a valuable lesson. It learns that producing a high-confidence prediction that is likely to be contradicted by another modality is "expensive" in terms of the overall loss. Therefore, it adjusts its weights to become more cautious and better calibrated. The unimodal network learns to associate maximum confidence not just with strong internal features, but with features that are also robust and likely to lead to cross-modal agreement. As shown in \cref{fig:image1} and \cref{fig:image2}, training with ACL effectively alleviates unimodal overconfidence on incorrect predictions across all datasets.



\begin{figure}[t]
    \centering
    \begin{minipage}{0.48\textwidth}
        \centering
        \includegraphics[width=\linewidth]{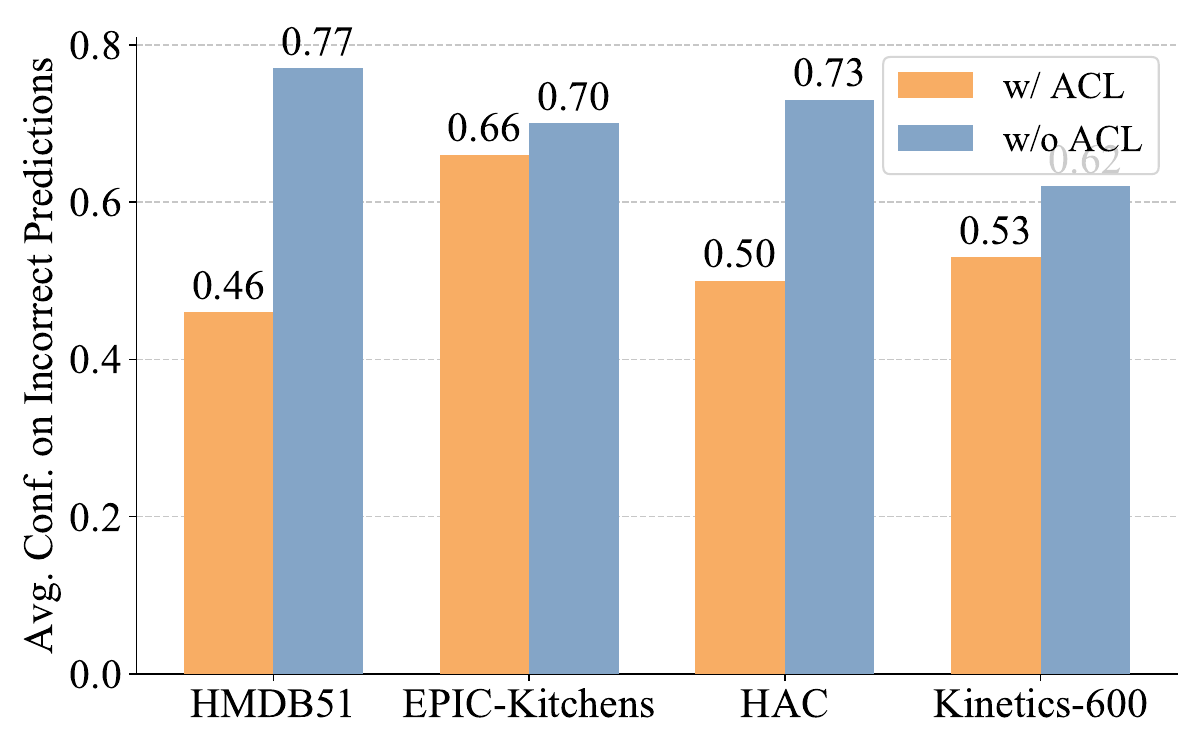}
        \vspace{-0.8cm}
        \caption{The average confidence on incorrect predictions for video modality w/ and w/o ACL.}
        \label{fig:image1}
    \end{minipage} \hfill
    \begin{minipage}{0.48\textwidth}
        \centering
        \includegraphics[width=\linewidth]{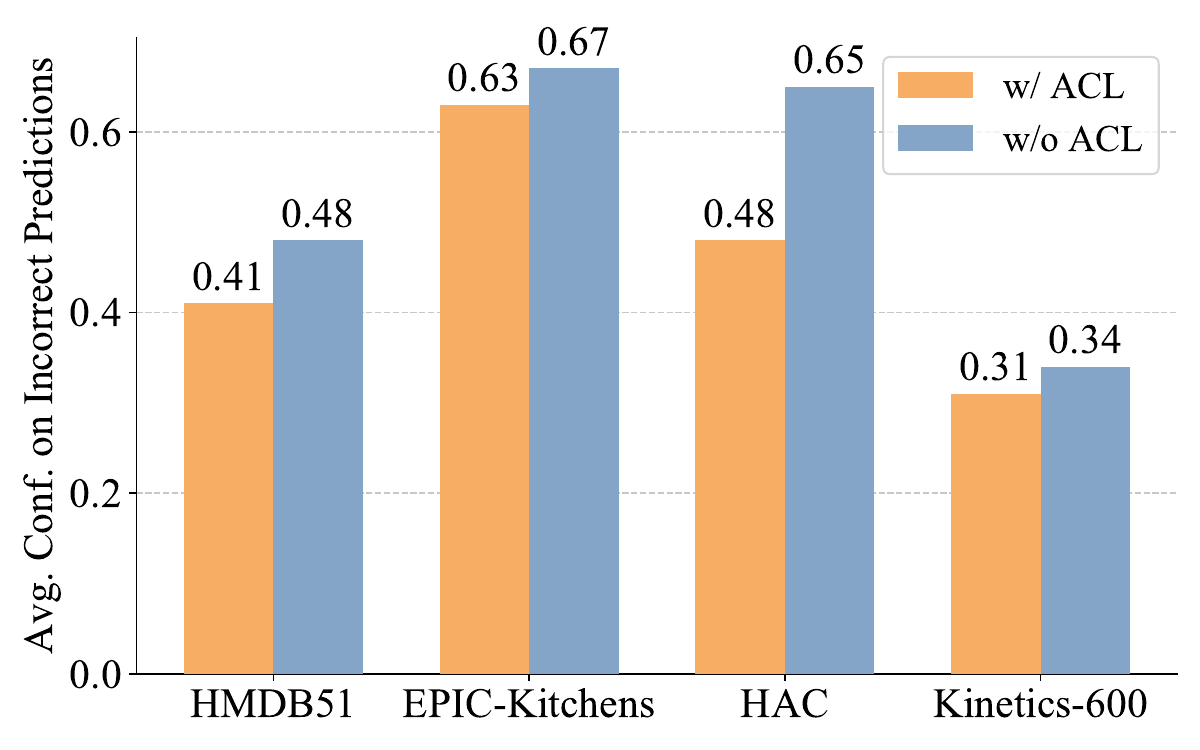}
        \vspace{-0.8cm}
        \caption{The average confidence on incorrect predictions for optical flow modality w/ and w/o ACL.}
        \label{fig:image2}
    \end{minipage}
\end{figure}

\subsection{Mechanism of MFS}
\noindent\textbf{Multimodal Feature Swapping.} Our MFS module isn't designed to replicate the entire, complex distribution of all possible real-world failures, as this would be intractable. Instead, MFS serves as a principled and targeted regularizer that exposes the model to a critical and common failure mode in multimodal systems: cross-modal inconsistency. Our core hypothesis is that many real-world errors arise from conflicting or ambiguous signals between modalities. MFS directly simulates this failure mode by swapping feature segments.
Unlike generic outlier exposure (OE) or manifold mixup, MFS creates challenging, near-ID hard negatives by preserving intra-modal semantics while breaking cross-modal consistency. Ultimately, the model learns \textit{which structurally ambiguous cases should stay low-confidence}.

The feature swapping operation preserves local modality structure within unswapped segments, maintaining realistic intra-modal patterns while disrupting cross-modal alignment through swapped segments, simulating realistic sensor conflicts or temporal misalignments. It can also generate outliers of varying difficulty by adjusting the $n_{\text{swap}}$ parameter, ranging from subtle inconsistencies when $n_{\text{swap}}$ is small to severe cross-modal conflicts when $n_{\text{swap}}$ is large.

Training on these synthesized samples forces the model to develop a more robust fusion mechanism that must critically assess the semantic agreement between modalities. This closer examination of cross-modal consistency improves its ability to detect real-world misclassifications, which often exhibit similar signal conflicts. The effectiveness of this approach is validated by our extensive empirical results, which demonstrate that learning to reject these synthetic inconsistencies directly translates to better identification of real-world errors. \cref{fig:mfs_illus} gives a detailed illustration of how MFS enables the detection and rejection of uncertain predictions, thereby improving model reliability. The soft labeling parameter $\lambda$ provides calibrated training signals proportional to the degree of synthetic corruption, allowing the model to learn fine-grained uncertainty estimation.

\noindent\textbf{Why Contiguous Swapping and Effective?} Real multimodal failures are typically localized/structured rather than i.i.d. white noise (e.g., a few seconds of missing audio, a region of blurred video, or temporal misalignment). Swapping contiguous blocks provides a simple inductive bias that introduces localized disruption instead of globally destroying the representation. Importantly, this choice is \textit{empirically validated}: contiguous swapping outperforms random perturbations (noise/drop/random mixing) in our ablations in Tab.~7.

\noindent\textbf{Semantic Preservation.}
MFS swaps only a subset of feature dimensions, leaving most dimensions unchanged and thereby \textit{preserving intra-modal semantics}. 
As evidence, predictions from original and MFS features agree in 99.2\% of cases \textit{(indicating preserved semantics)}, with the MFS features yielding {$0.1$ lower confidence on average}.

\noindent\textbf{Why MFS Succeeds where OE Fails?} OE mainly trains rejection under semantic shift (OOD). In contrast, \textit{multimodal failures often occur within ID semantics} but with conflicting/ambiguous evidence across modalities. MFS explicitly synthesizes such cross-modal inconsistency, producing samples that remain semantically ID yet structurally unreliable. 



\begin{figure}[t]
  \centering
  \includegraphics[width=0.8\linewidth]{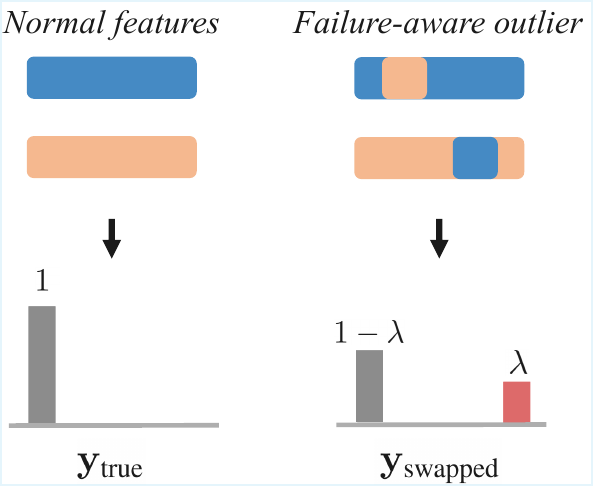}
  \vspace{-0.3cm}
  \caption{An illustration of how MFS enables the detection and rejection of uncertain predictions, thereby improving model reliability. The failure-aware outliers generated by MFS are treated as uncertain because they introduce cross-modal inconsistency. Accordingly, their soft label $\mathbf{y}_{\text{swapped}}$ reduces the ground-truth class probability from $1$ to $1-\lambda$, compared to $\mathbf{y}_{\text{true}}$. Training with cross-entropy loss on $\mathbf{y}_{\text{swapped}}$ ensures that the prediction confidence on failure-aware outliers remains lower than on normal features, which is exactly the desired outcome.}
  \label{fig:mfs_illus}
\end{figure}

\noindent\textbf{Difference between MFS and Feature Mixing.}
While both MFS and Feature Mixing~\citep{liu2025fm} generate outliers by swapping feature dimensions between modalities, their primary goal and supervision strategies are fundamentally different. MFS is designed for misclassification detection. Its goal is to teach the model to identify when it's making a mistake on in-distribution data. The model learns to classify the outlier using a specific soft label that mixes the true class with a new "outlier" class. This helps it recognize failure-aware samples that have conflicting internal signals. Feature Mixing is designed for out-of-distribution detection. Its goal is to teach the model to identify data that comes from unknown classes not seen during training. The model learns to be confused by the outlier. The entropy loss forces the model to be uncertain, outputting a flat, uniform probability distribution over all known classes. This teaches it to assign low confidence to anything that doesn't look like in-distribution data.

As for the implementation, MFS selects a single contiguous block of feature dimensions per modality (sample start indices $s_1$, $s_2$, and swap $n_{swap}$ consecutive dims). The size $n_{swap}$ is sampled from $\mathcal{U}(n_{\text{min}}, n_{\text{max}})$, so it can control how close and far the synthesized outliers lie from ID. Feature Mixing randomly samples an arbitrary subset of $N$ feature dimensions from each modality (non-contiguous, independent indices) and swaps those positions between modalities (simple index-based swap). Our ablation results show that MFS achieves superior performance compared to Feature Mixing in the context of multimodal failure detection.

\end{document}